\documentclass[pdflatex,sn-mathphys-num]{sn-jnl}

\usepackage{lmodern}
\usepackage{graphicx}%
\usepackage{multirow}%
\usepackage{amsmath,amssymb,amsfonts}%
\usepackage{amsthm}%
\usepackage{mathrsfs}%
\usepackage[title]{appendix}%
\usepackage{xcolor}%
\usepackage{subcaption}
\usepackage{textcomp}%
\usepackage{manyfoot}%
\usepackage{booktabs}%
\usepackage{algorithm}%
\usepackage{algorithmicx}%
\usepackage{algpseudocode}%
\usepackage{listings}%
\usepackage{graphicx}
\usepackage{epstopdf, epsfig}
\usepackage{mathrsfs}
\usepackage{amsfonts}
\usepackage{cleveref}
\usepackage{booktabs}
\usepackage[utf8]{inputenc}
\usepackage{algorithm}
\usepackage{algpseudocode}
\usepackage{tikz}
\usepackage{multimedia}
\usepackage{xspace}
\usepackage{amsthm}
\usepackage{graphicx}
\usepackage{bm}
\usepackage[utf8]{inputenc}
\usepackage[T1]{fontenc}
\usepackage{mathptmx}
\usepackage{newtxtext,newtxmath}
\usepackage{bbding}
\usepackage{hyperref}

\newcommand{\FNO}{\textsc{\small{FNO}}\xspace}
\newcommand{\MFN}{\textsc{\small{MFN}}\xspace}

\newcommand{\alg}{\textsc{\small{PINO-PC}}\xspace}
\newcommand{\modelname}{PINO-PC\xspace}
\newcommand{\DDPG}{\textsc{\small{DDPG}}\xspace}
\newcommand{\DNSPC}{\textsc{\small{DNS-PC}}\xspace}
\newcommand{\MPCNN}{\textsc{\small{MP-CNN}}\xspace}

\newcommand{\observername}{$M_O$}
\newcommand{\policyname}{$M_p$}
\newcommand{\re}{\operatorname{Re}}
\newcommand{\dui}{\CheckmarkBold}
\newcommand{\budui}{\XSolidBrush}
\newtheorem{theorem}{Theorem}

\newcommand{\zelinrevision}[1]{{} #1}

\begin{document}

\title[Physics-informed Neural-operator Predictive Control for Drag Reduction in Turbulent Flows]{Physics-informed Neural-operator Predictive Control for Drag Reduction in Turbulent Flows}


\author[1, 6]{\fnm{Zelin} \sur{Zhao}}

\author[1]{\fnm{Zongyi} \sur{Li}}

\author[1]{\fnm{Kimia} \sur{Hassibi}}

\author[2]{\fnm{Kamyar} \sur{Azizzadenesheli}}

\author[3]{\fnm{Junchi} \sur{Yan}}

\author[4]{\fnm{H. Jane} \sur{Bae}}

\author[4,5]{\fnm{Di} \sur{Zhou}}

\author[1]{\fnm{Anima} \sur{Anandkumar}}

\affil[1]{\orgdiv{Department of Computing and Mathematical Sciences}, \orgname{California Institute of Technology}, \orgaddress{\city{Pasadena}, \postcode{91125}, \state{CA}, \country{USA}}}

\affil[2]{\orgname{NVIDIA}, \orgaddress{\city{Pasadena}, \postcode{91125}, \state{CA}, \country{USA}}}

\affil[3]{\orgdiv{Department of Computer Science and Engineering}, \orgname{Shanghai Jiao Tong University}, \orgaddress{\postcode{200240}, \state{Shanghai}, \country{China}}}

\affil[4]{\orgdiv{Graduate Aerospace Laboratories}, \orgname{California Institute of Technology}, \orgaddress{\city{Pasadena}, \postcode{91125}, \state{CA}, \country{USA}}}

\affil[5]{\orgdiv{Department of Mechanical and Aerospace Engineering}, \orgname{University of Tennessee}, \orgaddress{\city{Knoxville}, \postcode{37996}, \state{TN}, \country{USA}}}

\affil[6]{\orgdiv{Daniel Guggenheim School of Aerospace Engineering}, \orgname{Georgia Institute of Technology}, \orgaddress{\city{Atlanta}, \postcode{30332}, \state{GA}, \country{USA}}}


\abstract{Assessing turbulence control effects for wall friction numerically is a significant challenge since it requires expensive simulations of turbulent fluid dynamics. We instead propose an efficient deep reinforcement learning (RL) framework for modeling and control of turbulent flows. It is model-based RL for predictive control (PC), where both the policy and the observer models for turbulence control are learned jointly using Physics Informed Neural Operators (PINO), which are discretization invariant and can capture fine scales in turbulent flows accurately. Our PINO-PC outperforms prior model-free reinforcement learning methods in various challenging scenarios where the flows are of high Reynolds numbers and unseen, i.e., not provided during model training. We find that PINO-PC achieves a drag reduction of 39.0\% under a bulk-velocity Reynolds number of 15,000, outperforming previous fluid control methods by more than 32\%.}




\maketitle
\section{Introduction}
\label{sec-intro}

Turbulent flows are prevalent in many areas of science and engineering, such as atmospheric weather~\cite{dutton1970clearAIRTurbulence}, ocean currents~\cite{oceanTurbulence}, and blood flow in arteries~\cite{arterialTurbulence} and veins~\cite{VeinTurbulence}. 
The turbulent flow is generally more unstable when compared to laminar flow and has a higher skin friction drag, which is caused by the friction of a fluid moving against a surface or a wall. Reducing such drag and controlling turbulent flows is necessary for various applications such as aerospace engineering, fluid transport, and biomedical devices~\cite{ClosedLoopTurbulenceControl}, since it mitigates adverse effects associated with turbulence, such as increased energy consumption, reduced efficiency, and heightened mechanical stress on structures. By gaining a deeper understanding of turbulent flow dynamics and implementing effective control strategies, we can enhance performance, optimize design, and ensure the safety and reliability of various engineering systems and biological processes. 

\begin{figure}[t]
    \centering 
    \begin{subfigure}{0.47\linewidth}
    \includegraphics[width=\textwidth]{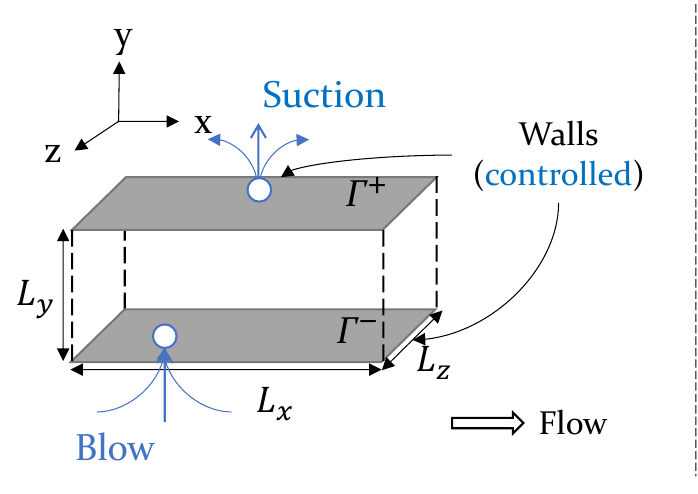}
    \vspace{-0.7cm}
    \caption{}
    \label{Fig-schematic}
    \end{subfigure}
    \begin{subfigure}{0.50\linewidth}
    \includegraphics[width=\textwidth]{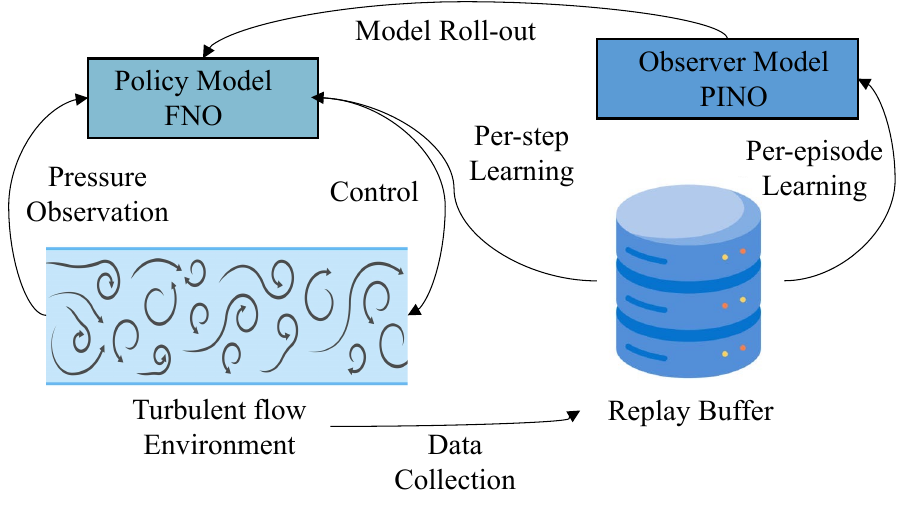}
    \vspace{-0.7cm}
    \caption{}
    \label{Fig-model-flow}
    \end{subfigure}
    \caption{\textbf{(a)}: Channel flow moving along the streamwise direction $x$, with the control applied at the wall via suction or blowing. \textbf{(b)}: Overall schematic of~\alg. \modelname consists of two neural operator components: a Physics-Informed Neural Operator (PINO~\cite{PINO}) that serves as the observer model for flow prediction, and a Fourier Neural Operator (FNO~\cite{FNO}) that acts as the policy model for control action generation. The controller takes pressure observations as input, performs predictive control, and applies the control to the turbulent flow environment. The PINO observer integrates both data and physics-based losses during training to learn accurate flow dynamics, while the FNO policy model optimizes control actions to minimize drag. Data is collected from the turbulent flow environment and is used to train the model.}
    \label{fig-scenario} 
\end{figure}

The standard approach to controlling turbulent channel flows~\cite{OppositionControl, lee1998suboptimal,DNSOptimalBenchmark,MLoppositionControl} involves blowing and suction of fluid air at chosen positions along the wall to control the boundary velocity (as shown in the left of~\Cref{Fig-schematic}). There are both passive and active control approaches. Passive approaches do not have a feedback loop for obtaining information and employing them for controlling flows~\cite{ho1982subharmonics,lowFreqUnsteadiness}, and hence, are generally inferior in drag reduction compared to active control methods~\cite{OppositionControl,lee1998suboptimal,DNSOptimalBenchmark}, which have a feedback loop and dynamically reduce skin friction through suction and blowing at the wall. 

A simple active control method in wall-bounded flows, known as ``opposition control''~\cite{OppositionControl}, proposes applying blowing and suction at the wall, opposite to the normal velocity at a plane off the wall called a ``detection plane''.  However, such an approach is  myopic and not an optimal control policy. Even when it achieves a reasonable drag reduction, it requires placement of sensors at the detection plane, which is often impractical~\citep{lee1998suboptimal, MLoppositionControl}. To overcome this, a machine-learning-based approach \MPCNN is introduced~\citep{MLoppositionControl} to predict the velocity function on the detection plane based on boundary information (e.g., boundary pressure function). This allows for sensor-free implementation by replacing direct velocity measurements with learned estimates. However, this method still operates within the framework of opposition control, meaning it does not fundamentally optimize the control policy but rather seeks to replicate an existing heuristic approach. Additionally, it often assumes the availability of a well-defined state-space dynamical model for controller synthesis, which may not always be practical or generalizable across different flow conditions.  

To address the above performance limitation of opposition control, reinforcement learning (RL) methods are developed that control the flow to achieve drag reduction~\cite{DRLControlPDE}. \zelinrevision{In particular, deep deterministic policy gradient (\DDPG)~\cite{DDPG} has been employed to control flows~\cite{sonoda2023reinforcement}, achieving superior drag reduction compared to opposition control~\citep{DeepRLchannel}.} It controls the flow by changing the mass flow rates of two jets on the sides of a cylinder. ~\citet{DRLFlow} and~\citet{RLBluffBodySim} use RL methods to control the cylinder or bluff body flow. \citet{tang2020robustReynolds} proposes a smoothing technique to reduce the drag fluctuations while enabling the RL agent to generalize to unseen Reynolds numbers. Recently, \citet{LearnTwoDControl} proposes to transfer discovered two-dimensional controls to three-dimensional cylinder flows via reinforcement learning. \zelinrevision{However, these RL approaches often have large variances~\citep{DeepRLchannel} and have inferior performance when the flow is of a high turbulent level (at a high Reynolds number)~\citep{lale2021model}.  This is due to several reasons such as using a model-free approach, assuming a fixed discretization and full observability of the dynamics, and not incorporating knowledge of physics that can cause unstable behavior in turbulent conditions. 
Our work overcomes these limitations by accurately modeling the fluid flow and dynamically controlling the turbulent conditions in an online manner.}



{\bf Our approach:} \zelinrevision{We propose physics-informed neural operator predictive control~(\alg),  a model-based deep reinforcement learning framework for drag reduction. It consists of two main components, the observer model and the policy model, as illustrated in~\Cref{Fig-model-flow}. The observer model predicts the control outcome, i.e., internal field velocity, based on the control, while the policy model is used to predict the control, which is the applied boundary velocity, based on the boundary pressure. \alg proceeds in multiple episodes. During each episode of \alg, the observer and policy models, learned so far, are kept fixed, and applied to interactively collect observations (pressure, velocity, and drag) from the flow environment.  These observations are stored in memory, known as the replay buffer in RL literature~\cite{dqn}.  During learning, the observations are retrieved from memory and used to update the observer model. The observer model is then kept fixed, and the policy model is updated. Note that our observer model is not fixed throughout all episodes of training the policy model, and hence, it incorporates different dynamics. Our observer can learn from the collected experiences of different controls, because it retains memory from prior episodes in the replay buffer.}

We consider learning in function spaces, while prior approaches RL approaches for drag reduction, assume a fixed discretization of pressure, velocity. In fluid dynamics, it is crucial to capture fine-scale features and high-resolution details to accurately predict fluid behavior~\cite{DNSOptimalBenchmark}. Recently, neural operators have been proposed for learning accurate fluid flow models in function spaces~\cite{azizzadenesheli2024neural}. Neural operators are an extension of standard neural networks and are discretization invariant, meaning they are not limited to one discretization or resolution. They can accurately approximate the solution operators of PDEs, such as fluid flow equations. The Fourier Neural Operator (FNO~\cite{FNO}) is a specific type of neural operator that leverages Fourier transforms to efficiently capture global dependencies in the solution space, making it particularly well-suited for fluid dynamics applications. Further, physics-informed neural operators (PINO) integrate training data with knowledge of physics into operator learning, such as equations as additional training supervision~\cite{PINO}. This reduces reliance on training data, which is crucial when data is scarce, and enables generalization to flows of unseen Reynolds numbers. In \alg, The observer model is trained under the PINO framework to minimize a combination of data and PDE losses, while the policy model is trained using the FNO model to minimize the control loss, which is kinetic energy and actuation norms on the trained observer predictive model and control cost function~\cite{lee1998suboptimal}.

{\bf Summary of empirical results: } \zelinrevision{Our numerical experiments show that \alg demonstrates a better drag reduction compared to previous machine-learning and reinforcement-learning approaches as well as traditional control methods that do not involve learning}. It achieves a $43.5\%$ drag reduction for flows with Reynolds numbers not included in the training data, which represents an improvement of 26.5\% in drag reduction when compared to prior learning approaches such as \MPCNN and 9.0\% when compared with \DDPG, a model free baseline~\cite{DDPG}. Our approach also outperforms control methods that do not involve learning, such as opposition control~\cite{OppositionControl} by $24.9\%$ and optimal control ~\citep{lee1998suboptimal,DNSOptimalBenchmark} by $9.6\%$ .  Furthermore, \alg leverages physics-informed learning, which boosts its generalization to unseen flows. The experimental results show that the generalization performance can improve drag reduction performance up to $2.2\%$ when using physics-informed learning compared to \alg without physics-informed learning.


%

\zelinrevision{Since \alg is the first model-based RL method for drag reduction, it has better generalization capabilities to new unseen environments, compared to model-free RL approaches proposed earlier. Further, since our online learning is physics-informed and incorporates PDE constraints, it can more easily generalize to new conditions such as fluid flows with new Reynolds numbers, especially high Reynolds numbers with highly turbulent dynamics, where control becomes harder, and prior RL methods fail. Fluid flows with different Reynolds numbers indeed have shared features at multiple scales that help with generalization to unseen scenarios. Even then, adapting the control to unseen Reynolds numbers, especially higher Reynolds numbers, is challenging due to increased nonlinear interactions.  Our method works effectively even under this challenging setting since it can adapt online to unseen scenarios, since it is physics-informed,  while also utilizing the shared features from its earlier training due to operator learning. Such transfer learning across different Reynolds numbers can be further enhanced by explicitly incorporating relationships across different scales, which is of interest for further investigation.}

Thus, our approach has demonstrated superior accuracy and drag reduction compared to alternative machine-learning methods. Notably, \alg achieves a remarkable 43.5\% drag reduction for Reynolds numbers not included in the training data, surpassing both opposition control and the optimal control baseline. The proposed iterative learning procedure, with extensive observer and policy learning, proves effective in achieving more robust turbulence control. This work provides a foundation for more efficient and practical turbulence control methodologies.

\section{Results}
We perform the direct numerical simulation (DNS) of a turbulent channel flow~\cite{lee1998suboptimal, OppositionControl, MLoppositionControl,DeepRLchannel}. The schematic of the channel flow is presented in~\Cref{Fig-schematic}. The simulations are based on discretizing the incompressible Navier-Stokes equations, while the equations are solved with an explicit third-order Runge Kutta (RK3) method for time advancement. The control is deployed by applying a normal velocity at the wall, while the control target is to minimize the drag. More details of the problem setting can be found in~\Cref{sec-problem-setting}.

\zelinrevision{We experiment with flows of various Reynolds numbers to show the drag reduction results of different approaches, where settings are detailed in the first row of~\Cref{tab-min-channel-data}. In the first setup (first two columns), the training and testing bulk-velocity Reynolds number is $\re_b\approx3k$. In this scenario, the friction Reynolds number is around $\re_\tau\approx178$, which is close to the default setup of previous studies~\cite{OppositionControl, DeepRLchannel, MLoppositionControl, bae2021nonlinear}. We then conduct various experiments on other choices of Reynolds numbers to test the generalization performance of different machine learning models. To reduce the effect of noises and randomness, we conduct each experiment with three different flow initializations and show mean and standard errors in tables and curve plots. We use different flow initializations in training and testing to assess the generalization performance of machine learning models.}

\begin{figure}
    \centering 
    \includegraphics[width=0.90\textwidth]{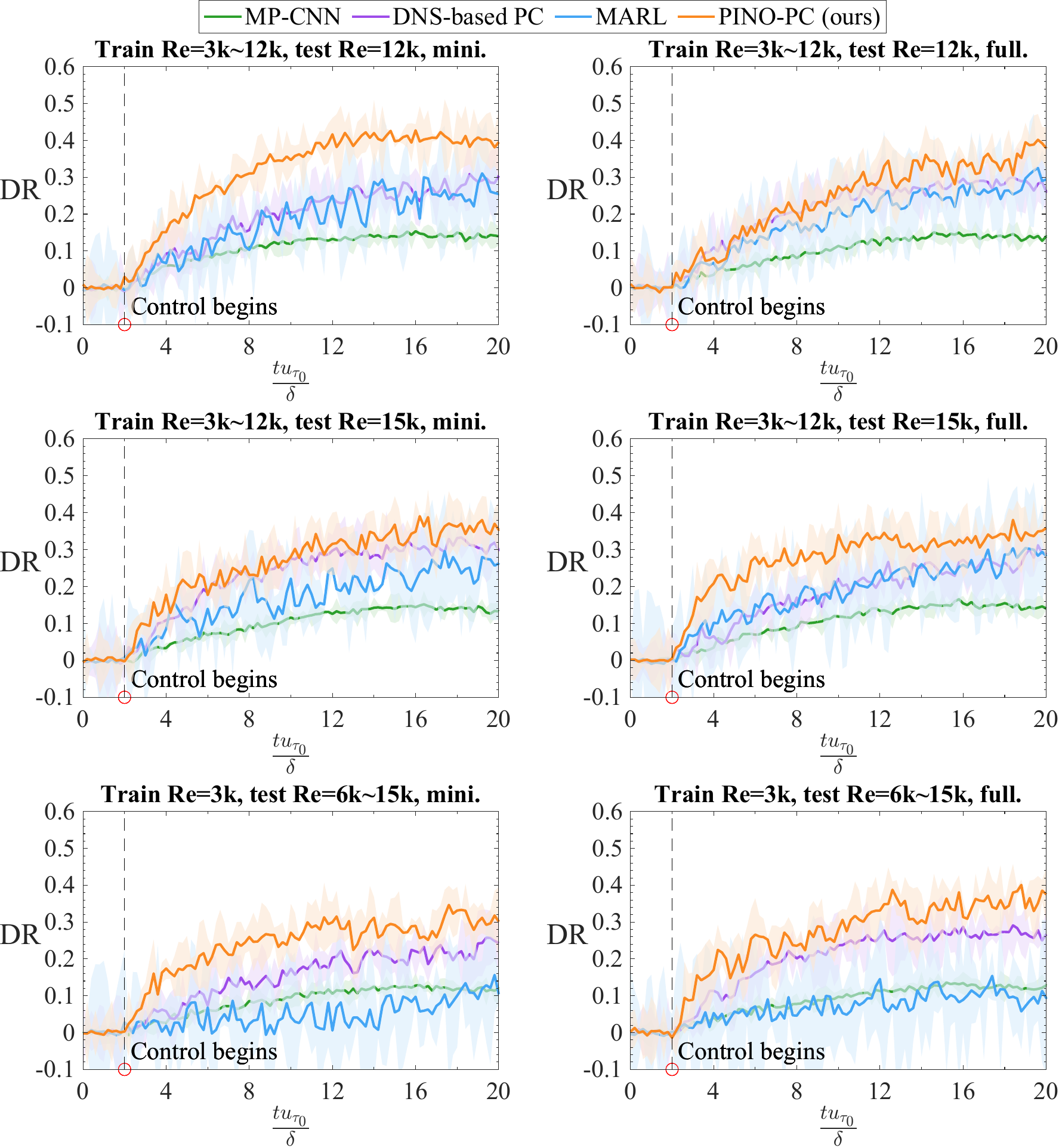}
    \caption{Drag reduction curves comparing several flow control methods: \MPCNN~\cite{MLoppositionControl}, Local suboptimal~\cite{lee1998suboptimal}, \DDPG~\cite{DeepRLchannel} and~\alg (ours). The $x$-axis denotes the non-dimensional timestep, and the $y$-axis denotes the drag reduction rate (DR) related to the uncontrolled case. The control beginning time is indicated via a red circle.}
    \label{fig-compare-curve}
\end{figure}

\begin{table*}
\def~{\hphantom{0}}
\resizebox{\columnwidth}{!}{%
\begin{tabular}{c|cc|cc|cc|cc}
\toprule
$\re_b$                          & \multicolumn{1}{c|}{3k}                                & 3k                                & \multicolumn{1}{c|}{3k, 6k, 9k, 15k}                   & 12k                               & \multicolumn{1}{c|}{3k, 6k, 9k, 12k}                   & 15k                               & \multicolumn{1}{c|}{3k}                                & 6k, 9k, 12k, 15k                  \\ \midrule
Phase                                          & \multicolumn{1}{c|}{Training}                          & Testing                           & \multicolumn{1}{c|}{Training}                          & Testing                           & \multicolumn{1}{c|}{Training}                          & Testing                           & \multicolumn{1}{c|}{Training}                          & Testing                           \\ \midrule
Opposition control~\citep{OppositionControl} & \multicolumn{1}{c|}{-}                                 & $17.2\pm\scriptstyle 1.5$         & \multicolumn{1}{c|}{-}                                 & $16.3\pm\scriptstyle 1.2$         & \multicolumn{1}{c|}{-}                                 & $16.6\pm\scriptstyle 2.1$         & \multicolumn{1}{c|}{-}                                 & $15.1\pm\scriptstyle 1.5$         \\ \midrule
\DNSPC~\cite{DNSOptimalBenchmark}        & \multicolumn{1}{c|}{-}                                 & $38.1\pm\scriptstyle 3.5$         & \multicolumn{1}{c|}{-}                                 & $31.1\pm\scriptstyle 4.8$         & \multicolumn{1}{c|}{-}                                 & $33.1\pm\scriptstyle 4.7$         & \multicolumn{1}{c|}{-}                                 & $25.7\pm\scriptstyle 3.8$         \\ \midrule
\MPCNN~\citep{MLoppositionControl}             & \multicolumn{1}{c|}{$15.3\pm\scriptstyle2.5$}          & $15.1\pm\scriptstyle2.6$          & \multicolumn{1}{c|}{$16.1\pm\scriptstyle3.1$}          & $14.2\pm\scriptstyle 1.4$         & \multicolumn{1}{c|}{$15.4\pm\scriptstyle 3.4$}         & $15.4\pm\scriptstyle 1.6$         & \multicolumn{1}{c|}{$15.3\pm\scriptstyle 2.5$}         & $13.1\pm\scriptstyle 1.2$         \\ \midrule
\DDPG~\cite{DeepRLchannel}                      & \multicolumn{1}{c|}{$41.2\pm\scriptstyle7.1$}          & $40.5\pm\scriptstyle7.8$          & \multicolumn{1}{c|}{$36.1\pm\scriptstyle5.6$}          & $32.1\pm\scriptstyle7.3$          & \multicolumn{1}{c|}{$44.2\pm\scriptstyle7.9$}          & $31.4\pm\scriptstyle7.7$          & \multicolumn{1}{c|}{$41.2\pm\scriptstyle7.1$}          & $14.6\pm\scriptstyle9.3$          \\ \midrule
\alg                                     & \multicolumn{1}{c|}{$\textbf{45.1}\pm\scriptstyle6.7$} & $\textbf{45.5}\pm\scriptstyle5.4$ & \multicolumn{1}{c|}{$\textbf{45.3}\pm\scriptstyle4.0$} & $\textbf{41.2}\pm\scriptstyle4.6$ & \multicolumn{1}{c|}{$\textbf{42.1}\pm\scriptstyle5.7$} & $\textbf{38.5}\pm\scriptstyle6.4$ & \multicolumn{1}{c|}{$\textbf{45.1}\pm\scriptstyle6.7$} & $\textbf{33.5}\pm\scriptstyle5.5$ \\ \bottomrule
\end{tabular}}
  \caption{Performances in varied Reynolds numbers comparing several flow control methods in the \textbf{minimum channel flow} case. The metric is drag reduction rate (DR) in percentage. Each experiment is repeated three times while we report both mean and variance in this table. Opposition control~\cite{OppositionControl} and \DNSPC~\cite{DNSOptimalBenchmark} do not have training performance scores because they are not machine-learning-based methods. \zelinrevision{We experimented with different generalization settings, where corresponding Reynolds numbers are presented in the first row.}}
  \label{tab-min-channel-data}
\end{table*}

\begin{table*}
\def~{\hphantom{0}}
\resizebox{\columnwidth}{!}{%
\begin{tabular}{c|cc|cc|cc|cc}
\toprule
$\re_b$                          & \multicolumn{1}{c|}{3k}                                 & 3k                                 & \multicolumn{1}{c|}{3k, 6k, 9k, 15k}          & 12k                               & \multicolumn{1}{c|}{3k, 6k, 9k, 12k}           & 15k                               & \multicolumn{1}{c|}{3k}                        & 6k, 9k, 12k, 15k                  \\ \midrule
Phase                                          & \multicolumn{1}{c|}{Training}                           & Testing                            & \multicolumn{1}{c|}{Training}                 & Testing                           & \multicolumn{1}{c|}{Training}                  & Testing                           & \multicolumn{1}{c|}{Training}                  & Testing                           \\ \midrule
Opposition control~\citep{OppositionControl} & \multicolumn{1}{c|}{-}                                  & $17.4\pm\scriptstyle 1.4$          & \multicolumn{1}{c|}{-}                        & $15.2\pm\scriptstyle 1.9$         & \multicolumn{1}{c|}{-}                         & $15.8\pm\scriptstyle 2.9$         & \multicolumn{1}{c|}{-}                         & $14.1\pm\scriptstyle 1.9$         \\ \midrule
\DNSPC~\cite{DNSOptimalBenchmark}        & \multicolumn{1}{c|}{-}                                  & $40.3\pm\scriptstyle 3.4$          & \multicolumn{1}{c|}{-}                        & $30.2\pm\scriptstyle 5.4$         & \multicolumn{1}{c|}{-}                         & $30.5\pm\scriptstyle 4.1$         & \multicolumn{1}{c|}{-}                         & $29.4\pm\scriptstyle 3.9$         \\ \midrule
\MPCNN~\citep{MLoppositionControl}             & \multicolumn{1}{c|}{$15.8\pm\scriptstyle 2.3$}          & $15.6\pm\scriptstyle 2.4$          & \multicolumn{1}{c|}{$18.4\pm\scriptstyle3.2$} & $15.2\pm\scriptstyle 1.6$         & \multicolumn{1}{c|}{$15.9\pm\scriptstyle 3.5$} & $16.1\pm\scriptstyle 1.4$         & \multicolumn{1}{c|}{$15.8\pm\scriptstyle 2.3$} & $13.4\pm\scriptstyle 2.0$         \\ \midrule
\DDPG~\cite{DDPG}                      & \multicolumn{1}{c|}{$34.1\pm\scriptstyle 6.9$}          & $33.1\pm\scriptstyle 5.9$          & \multicolumn{1}{c|}{$36.2\pm\scriptstyle5.2$} & $31.4\pm\scriptstyle7.0$          & \multicolumn{1}{c|}{$38.2\pm\scriptstyle8.1$}  & $32.5\pm\scriptstyle7.9$          & \multicolumn{1}{c|}{$34.1\pm\scriptstyle6.9$}  & $14.6\pm\scriptstyle9.1$          \\ \midrule
\alg                                     & \multicolumn{1}{c|}{$\textbf{43.5}\pm\scriptstyle 3.9$} & $\textbf{42.1}\pm\scriptstyle 4.9$ & \multicolumn{1}{c|}{$\textbf{43.1}\pm\scriptstyle3.9$} & $\textbf{40.3}\pm\scriptstyle3.2$ & \multicolumn{1}{c|}{$\textbf{40.1}\pm\scriptstyle6.2$}  & $\textbf{35.1}\pm\scriptstyle5.9$ & \multicolumn{1}{c|}{$\textbf{43.5}\pm\scriptstyle3.9$}  & $\textbf{39.0}\pm\scriptstyle4.0$ \\ \bottomrule
\end{tabular}}
  \caption{Performances in varied Reynolds numbers comparing several flow control methods in the \textbf{full channel flow} case. Other setups are the same as \Cref{tab-min-channel-data}.}
  \label{tab-full-channel-data}
\end{table*}

~\Cref{tab-min-channel-data} presents numerical control results in the minimal channel. Opposition control~\cite{OppositionControl} reported a drag reduction of $\approx 14\%$. Our results indicate that opposition control reaches a drag reduction of $\approx 17.2\%$ in the single Reynolds setup, which is in agreement with their result. Opposition control~\cite{OppositionControl} has similar results in other Reynolds numbers under other settings. The machine-learning-based opposition control method called \MPCNN~\cite{MLoppositionControl} has lower drag reduction rates in all settings than the traditional opposition control~\cite{OppositionControl} because it does not perfectly imitate the opposition control policy. \DNSPC~\cite{DNSOptimalBenchmark} has a much better drag reduction than opposition control because it can access interior information and optimizes control for a period of time. Our reproduction shows that they can achieve a drag reduction of $\approx38.1\%$ in the minimum channel flow case, which is close to their report ($\approx 40\%$). \DNSPC~\cite{DNSOptimalBenchmark} has lower performances in higher Reynolds numbers. \DDPG~\cite{DDPG} is a strong baseline and performs highly in the single Reynolds number setup. However, when scaled to higher Reynolds numbers, it suffers from overfitting and cannot perform well in the test splits of generalization settings. Also, it has a larger variance than \MPCNN~\cite{MLoppositionControl}. We find \alg consistently outperforms other methods across different Reynolds numbers and generalization settings, achieving the highest mean drag reduction rates. In the test splits of challenging generalization setups, \alg performs much better than \DDPG~\cite{DDPG}, which indicates the effectiveness of the proposed physics-informed learning scheme in narrowing the generalization gap. We also observe that \alg has larger variances than \MPCNN~\cite{MLoppositionControl} because it also needs to interact with the flows during training, which increases the training variance. Further, \alg has a smaller variance than \DDPG~\cite{DDPG}, which reveals the effect of using a physics-informed observer model to lower training variance. \zelinrevision{In fluid dynamics, the dynamics at different Reynolds numbers share similar behavior at different scales~\cite {fukami2024data}. In PINO, this is exploited through PDE loss, which improves generalization. This can be further enhanced by explicitly incorporating relationships across different scales, which is left for future investigation. }

We also provide the experimental results of the full channel flow in~\Cref{tab-full-channel-data}. We observe that opposition control~\cite{OppositionControl} and \MPCNN~\cite{MLoppositionControl} achieve a similar result in full channel flow compared to the minimum channel case. This is because opposition-control-based methods are not affected much by scales~\cite{MLoppositionControl}, and they usually have a smaller actuation intensity~\cite{DeepRLchannel}. \DNSPC~\cite{DNSOptimalBenchmark} demonstrates competitive performances in the full channel flow. It does not behave much differently in the minimum channel flow case because it resolves the optimization problem and calculates the control policy under each scenario. \DDPG~\cite{DDPG}'s performance downgrades significantly in the full channel flow compared to the minimum one. One possible explanation is that their policy model based on fully connected networks (FCN) is sensitive to the flow scale. Nonetheless, \DDPG performs better than \MPCNN~\cite{MLoppositionControl} and opposition control~\cite{OppositionControl}. Furthermore, we observe that \alg also has strong performance in the full channel flow case, which is primarily due to the fact neural operators~\cite{neuralOperator} can scale to other input dimensions because they learn solutions in the function space.

~\Cref{fig-compare-curve} shows comparison curves under several setups. We use shadowed regions to denote variance in this plot, and the control beginning time is marked by a red circle in the plot. We observe that opposition control~\cite{OppositionControl} and \MPCNN~\cite{MLoppositionControl} have smaller training variances with poor control outcomes. \DNSPC~\cite{DNSOptimalBenchmark} performs strongly in some scenarios, especially in high Reynolds numbers. \DDPG~\cite{DeepRLchannel} has large training variances because deep RL algorithms often require extensive trial-and-error, and unsuccessful explorations fail to bring drag reduction. By contrast, \alg has a smaller training variance than \DDPG~\cite{DeepRLchannel} and performs better, especially in high Reynolds numbers, which suggests the benefits of adopting the physics-informed neural-operator-based observer. 

Our approach further uncovers the power of machine-learning-based methods in reducing drag. The advantage of our approach lies both in its ability to reduce drag effectively and in the physical policies it learns~\citep{sonoda2023reinforcement}. Specifically, our method optimizes control policies in a way that leads to dynamically adaptive flow modifications (through our observer design), some of which align with established drag-reduction mechanisms, while others introduce nontrivial patterns of actuation (through our learned neural operator-based policy). Note that, as is common in RL, the learned policy may learn non-trivial ways to reduce drag that, on its own, is a topic of further study.

\Cref{fig-velocity-profile} provides the flow statistics of adopting~\alg after control in the full channel flow case under a Reynolds number of $\re_b = 3k$ and $\re_\tau=178$. We observe that the velocity fluctuations in three dimensions decrease after control. The turbulent kinetic energy (TKE) also decreases after control via \alg, which verifies the effectiveness of the control algorithm from another perspective. 


\begin{figure}[htb]
    \centering 
    \includegraphics[width=0.9\textwidth]{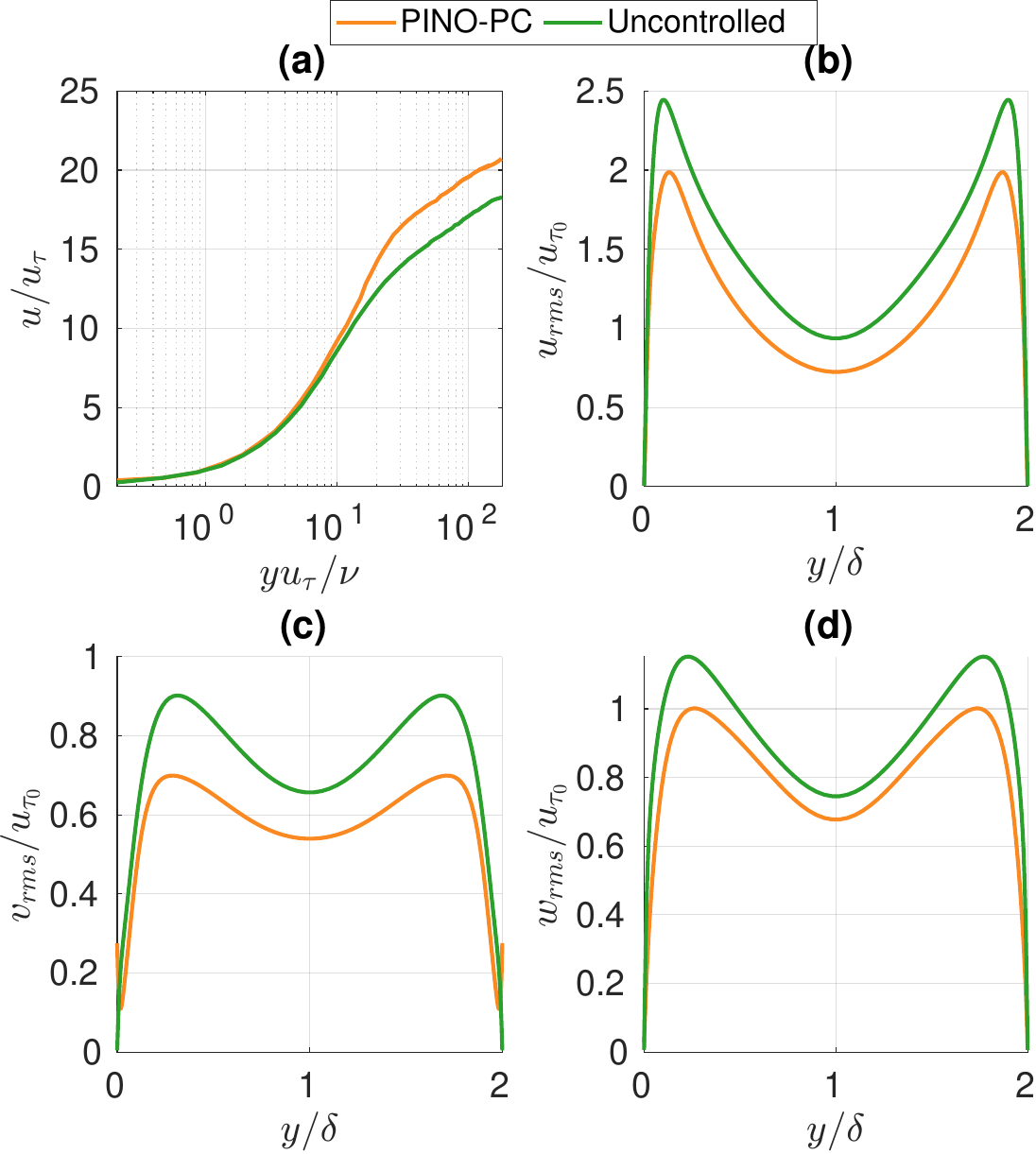}
    \caption{Time averaged statistics of the uncontrolled flow and~\alg (after control) in the full channel flow case. (a) The mean stream-wise velocity profile. (b, c, d) Stream-wise, wall-normal, and span-wise r.m.s. velocity fluctuations in the wall-normal direction.}
    \label{fig-velocity-profile}
\end{figure}

To gain further insight into the dynamics of the current control strategy, additional analyses of the velocity field have been conducted. \Cref{fig-vel-PDF} shows the joint probability density function (PDF) of streamwise ($u'$) and wall-normal ($v'$) fluctuating velocities at the location of $y^+= 15$, where the root-mean-square (r.m.s.) value of the streamwise velocity fluctuation reaches its maximum. The results are evaluated over the statistically steady period for both the uncontrolled and controlled cases at $Re_b\approx 3k$ or $Re_\tau\approx 178$. The joint PDF provides a statistical picture of how velocity fluctuations in the two directions are correlated and highlights the intensity and characteristics of turbulence-producing events, such as sweeps and ejections associated with near-wall streaks in channel flow \citep{jimenez2018coherent}. By normalizing the velocities with the friction velocity of the uncontrolled case, the changes in the distribution of velocity fluctuations introduced by the control become more apparent. The exhibited comparison of the PDFs at $y^+=15$ from the uncontrolled and controlled cases shows that the control alters the distribution of fluctuating velocities. Specifically, with the present control strategy, the range of fluctuations in both the streamwise and wall-normal directions is substantially reduced, suggesting that the near-wall streak intensity and associated sweep as well as ejection activity are weakened. The overall shape of the PDF remains similar between the two cases, with ejection and sweep events still dominating. These qualitative changes are consistent with the effects reported for opposition control \citep{DeepRLchannel}.


\begin{figure}[htb]
   \centering 
   \vspace{-3.5cm}
   \includegraphics[width=1.0\textwidth]{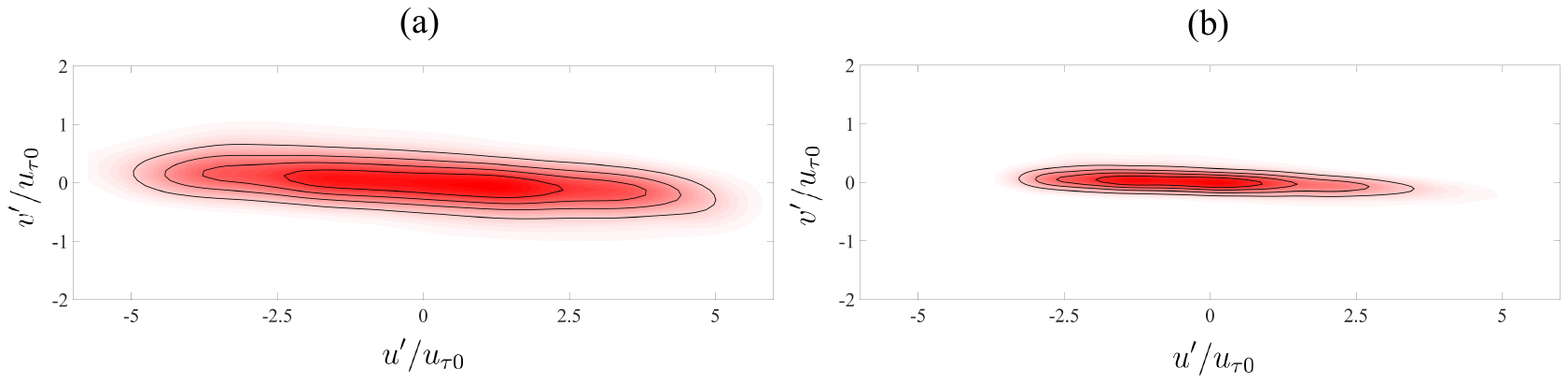}
   \vspace{-3cm}
   \caption{Joint probability density function of the streamwise and wall-normal velocity fluctuations at $y^+=15$ for (a) the uncontrolled full-channel flow and (b) the PINO-PC (with control) full-channel flow. The contour lines denote 20\%, 40\%, 60\%, and 80\% of the maximum probability density values.}
   \label{fig-vel-PDF}
\end{figure}

Beyond the joint PDF of fluctuating velocity components, the premultiplied energy spectra of the streamwise velocity fluctuations for both the uncontrolled and controlled cases are examined, and the results are shown in \Cref{fig-premult-spectra}. These spectra help quantify the impact of control on the characteristic scales of turbulent structures at different wall-normal locations in the channel. Here, $k_x$ and $k_z$ denote the streamwise and spanwise wavenumbers, respectively, while $\lambda_x$ and $\lambda_z$ represent the corresponding wavelengths. The spectra indicate that the flow is dominated by near-wall streaks. In the uncontrolled case, the dominant near-wall structures exhibit a streamwise length scale of $\lambda_x^+\approx500$ and a spanwise length scale of $\lambda_z^+\approx100$, concentrated around $y^+\approx 15$. In the controlled case, although the length scales of the dominant near-wall structures remain similar to those in the uncontrolled flow, their wall-normal positions are shifted slightly upward. More importantly, the energy level of these dominant structures is noticeably reduced, consistent with the trends observed in \Cref{fig-vel-PDF}.

\begin{figure}[htb]
   \centering 
   \includegraphics[width=1.0\textwidth]{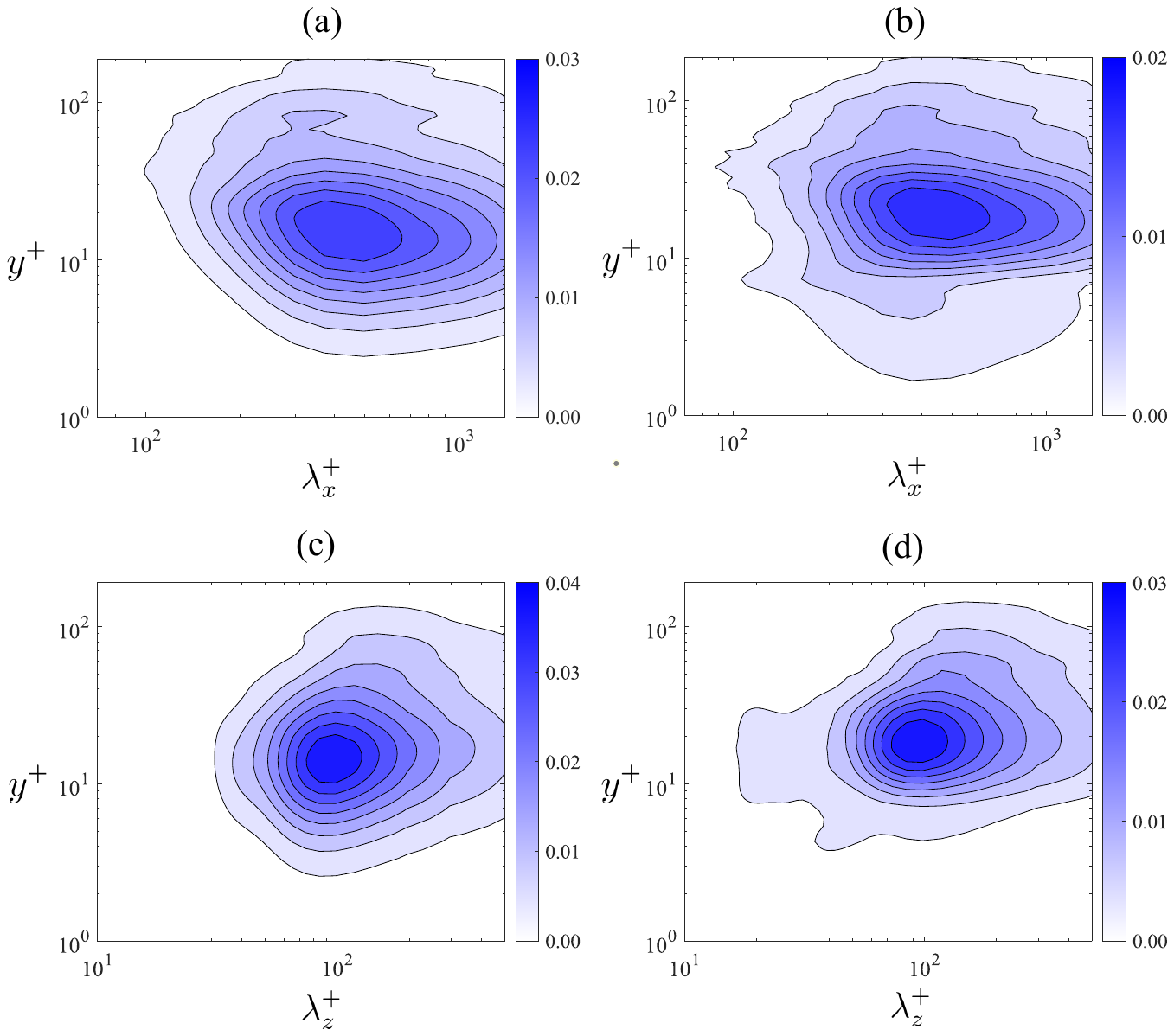}
   \caption{Premultiplied energy spectra of the streamwise velocity fluctuations $u'$ for (a, c) the uncontrolled full-channel flow and (b, d) the PINO-PC (with control) full-channel flow as a functions of wall-distance and wavelengths.}
   \label{fig-premult-spectra}
\end{figure}

The results of the joint PDF and premultiplied energy spectra of the velocity fluctuations indicate a substantial weakening of turbulence activity in the near-wall region, which is closely linked to skin-friction drag generation in wall-bounded flows. In particular, the decreased range of streamwise and wall-normal fluctuations, together with the reduction in energy of the dominant near-wall streaks and associated sweep and ejection events, reflects a reduction in both the intensity of streaks and the strength of vortical structures driving turbulence production, resulting in weaker Reynolds stresses. These findings suggest that the mechanisms responsible for turbulent momentum transfer toward the wall are significantly suppressed. Overall, the observed changes in velocity statistics and spectral content provide clear evidence that the current control strategy effectively attenuates near-wall turbulence structures, leading to the measured drag reduction.

\begin{figure}
    \centering 
    \begin{subfigure}{0.45\linewidth}
    \includegraphics[width=\textwidth]{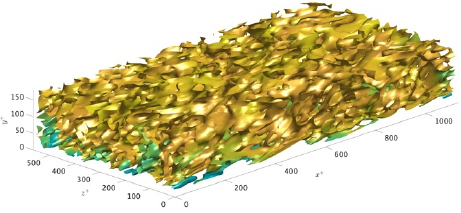}
    \label{Fig-vis3d-uncontrolled}
    \caption{Method: Uncontrolled, DR: $0.0$\%}
    \end{subfigure}
    \begin{subfigure}{0.46\linewidth}
    \includegraphics[width=\textwidth]{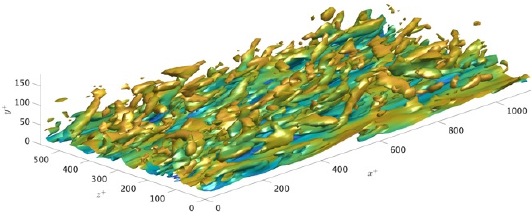}
    \label{Fig-vis3d-mp-cnn}
    \caption{Method: MP-CNN, DR: $15.6$\%}
    \end{subfigure}
    \begin{subfigure}{0.49\linewidth}
    \includegraphics[width=\textwidth]{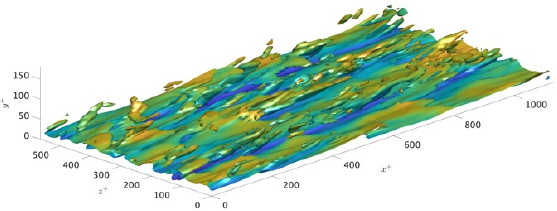}
    \label{Fig-vis3d-marl}
    \caption{Method: DDPG, DR: $33.1$\%}
    \end{subfigure}
    \begin{subfigure}{0.49\linewidth}
    \includegraphics[width=\textwidth]{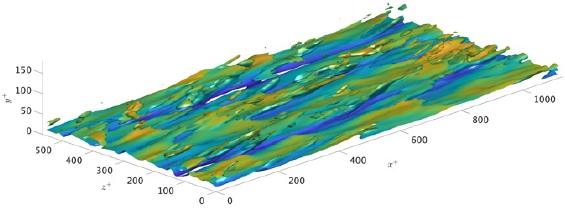}
    \label{Fig-vis3d-pino-pc}
    \caption{Method: PINO-PC, DR: $42.1$\%}
    \end{subfigure}
    \caption{Visualizations of the isosurface of the Q-criterion (a vorticity indicator)~\cite{blackburn1996topology, hammond1998observed,vortexIdentification} under the same threshold ($Q_\tau=50$) after control with Reynolds number $\re_u=3k$. The plane at $y^+=0$ represents the wall boundary. The isosurface is colored by the velocity magnitude. Four flows under four flow control methods are shown in this diagram: Uncontrolled, MP-CNN~\cite{MLoppositionControl}, DDPG~\cite{DeepRLchannel}, and our approach.  Please refer to the main text for details.}
    \label{fig-turbulence-3d-vis}
\end{figure}

\begin{figure}[htb]
    \centering 
    \includegraphics[width=0.94\textwidth]{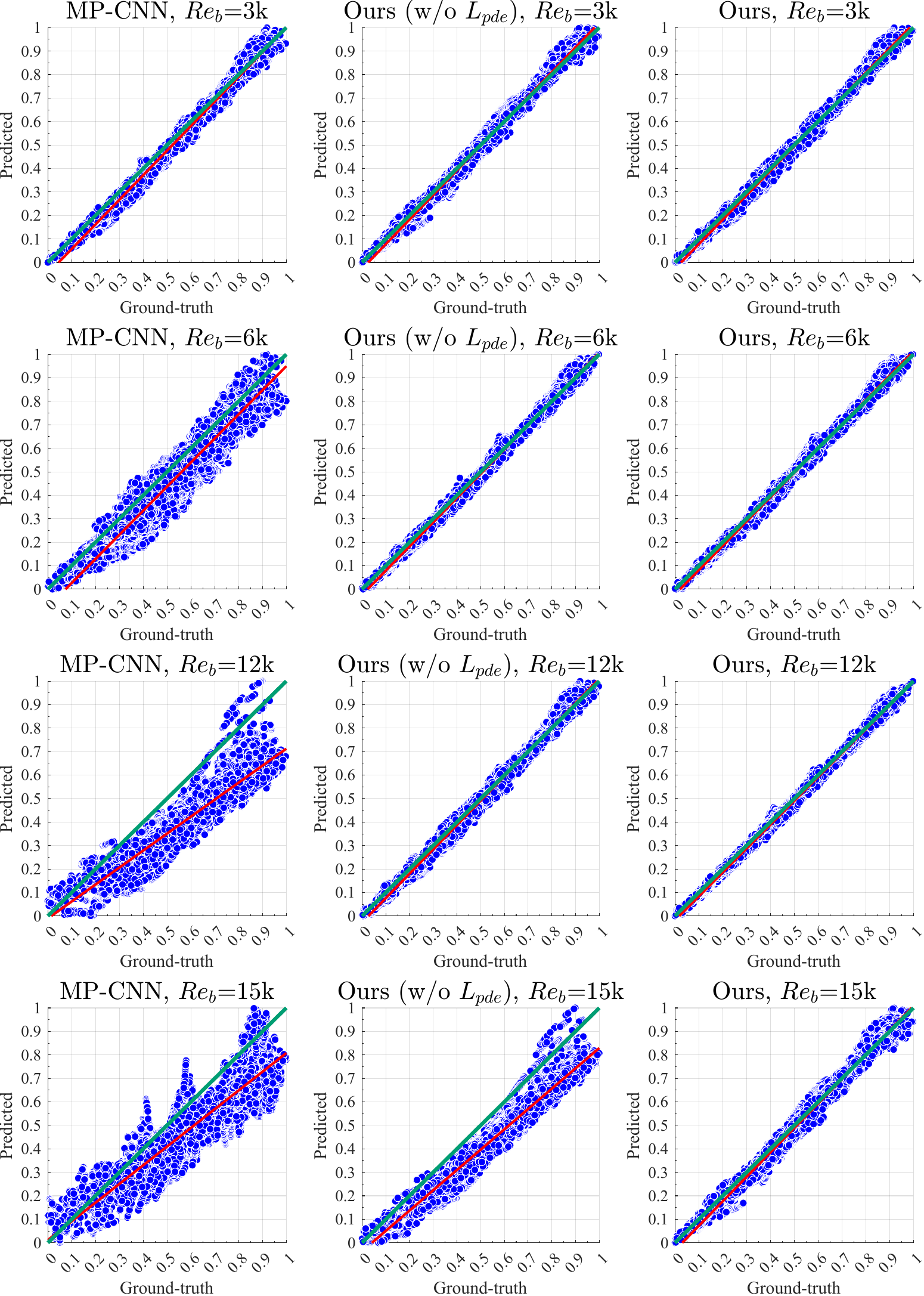}
    \caption{Scatter plot of predicted velocities (y-axis) and the gt velocities (x-axis) under different Reynolds numbers. We normalize the velocities into the [0, 1] range before plotting. We plot the $y=x$ line in green and the fitting line in red. The more accurate a model is, the closer the plotted points should surround the $y=x$ line.}
    \label{fig-scatter-plot}
\end{figure}

Finally, we conduct an isosurface visualization~\cite{hunt1988eddies, vortexIdentification,blackburn1996topology} to compare the performances of several control methods, and the result is presented in~\Cref{fig-turbulence-3d-vis}. The setup is also a full-channel flow under a Reynolds number of $\re_b = 3k$ and $\re_\tau=178$. The isosurface is formed by computing the Q-criterion~\cite{vortexIdentification}, which is associated with vorticity in the flow. The Q-criterion can be computed as:
\begin{equation}
Q \equiv \frac{1}{2}\left(\|\boldsymbol{\Omega}\|^2-\|\boldsymbol{S}\|^2\right),
\label{eq-q-criterion}
\end{equation}
where $\|\boldsymbol{S}\|=\left[\operatorname{tr}\left(\boldsymbol{S} \boldsymbol{S}^T\right)\right]^{1 / 2}$, $\|\boldsymbol{\Omega}\|=\left[\operatorname{tr}\left(\boldsymbol{\Omega} \boldsymbol{\Omega}^T\right)\right]^{1 / 2}$. Here, $\boldsymbol{S}$ and $\boldsymbol{\Omega}$ are the symmetric and anti-symmetric components of the velocity gradient tensor $\nabla u$. Thus, $\boldsymbol{S}$ is the rate of strain tensor, while $\Omega$ is the vorticity tensor. Therefore, $Q$ represents the difference between vorticity magnitude and shear strain rate. The more complex the isosurface of the Q-criterion is, the more vorticity is involved. We can observe from~\Cref{fig-turbulence-3d-vis} that \alg has the most simple isosurface when compared to baselines (Uncontrolled, \MPCNN~\cite{MLoppositionControl}, \DDPG~\cite{DeepRLchannel}). In this plot, we color the isosurface via the velocity magnitude. A large velocity corresponds to a yellow-colored surface, while a smaller velocity is associated with a blue-colored surface. We can observe from this plot that the middle region of the channel flow is associated with large velocities, while Q vanishes at the wall~\cite{vortexIdentification}.

\zelinrevision{Thus, the visualizations of the Q-criterion isosurfaces reveal that our method effectively modulates coherent structures over time, leading to a more stable and controlled flow compared to MP-CNN and DDPG. Specifically, our approach consistently suppresses excessive vortex generation while maintaining structured turbulence, as indicated by the velocity magnitude coloring. In contrast, MP-CNN and DDPG exhibit more fragmented or unstable structures, suggesting less effective regulation of turbulence. Additionally, our method demonstrates improved long-term stability, reducing chaotic fluctuations observed in the uncontrolled case. These results align with the physical intuition that effective control should mitigate high-vorticity regions while maintaining flow coherence. We have provided video visualizations of the time evolution of flow fields with our code release.}

\section{Methods}
\label{sec-methodology}
In this section, we first introduce the problem setup in~\Cref{sec-problem-setting}. We then introduce our proposed method, called physics-informed neural-operator-based predictive control (\alg), in several upcoming subsections. In~\Cref{sec-algorithm-outline}, we introduce the algorithm outline and overview of our proposed predictive control scheme. Subsequently, we propose details of two machine learning models adopted in our framework in~\Cref{sec-observer-model} and~\Cref{sec-policy-model}.

\subsection{Problem setting}
\label{sec-problem-setting}
In this work, we perform a direct numerical simulation (DNS) of a turbulent channel flow, which has been studied in previous drag reduction works~\cite{lee1998suboptimal, OppositionControl, MLoppositionControl,DeepRLchannel}. The flow domain is designed such that the $x$ direction indicates the streamwise direction while $y$ and $z$ direction denote the wall-normal and spanwise directions, respectively.

\subsubsection{The governing equation} 

The governing incompressible Navier-Stokes equations can be formulated as
\begin{equation}
\left\{
\begin{aligned}
\frac{\partial u_j}{\partial x_j} & =0, \\
\frac{\partial u_i}{\partial t}+u_j \frac{\partial u_i}{\partial x_j} & = -\frac{1}{\rho}\frac{\mathrm{d} P}{\mathrm{d} x_1} \delta_{1 i} -\frac{1}{\rho}\frac{\partial p}{\partial x_i}+\nu \frac{\partial^2 u_i}{\partial x_j \partial x_j},
\end{aligned}
\right.
\label{eq-pde}
\end{equation}
where $\delta_{ij}$ is the Kronecker delta, $\rho$ is the density, and $\nu$ is the kinematic viscosity.
Here $(x_1, x_2, x_3) = (x, y, z)$ is the position vector, $(u_1, u_2, u_3) = (u, v, w)$ is the corresponding velocity, and $-{\mathrm{d} P}/{\mathrm{d} x_1}$ is the applied mean pressure gradient to drive the flow. We use $u_{\tau_0}$ to denote the wall shear velocity of the uncontrolled flow, and we use $u_\tau$ to denote the wall shear velocity during the control. Note that here, we use a different term, $p$, written in lowercase to represent the pressure fluctuation. 
The friction Reynolds number of the flow is defined as $\re_\tau = u_{\tau_0}\delta / \nu$, where $\delta$ is the channel half height.

\subsubsection{The solver} 

The simulations are performed by discretizing the incompressible Navier–Stokes equations~(\Cref{eq-pde}) with a staggered, second-order-accurate, central finite-difference method in space and an explicit third-order-accurate Runge–Kutta method for time advancement \cite{wray1990minimal}. The system of equations is solved via an operator splitting approach \cite{chorin1968numerical}. This code has been validated by prior studies on turbulent channel flow \citep{bae_2018, bae_2019, bae2021nonlinear}. The computation domain of the simulation is denoted by $\Omega$, and we use $\Gamma^{+}$ and $\Gamma^{-}$ to denote two walls located at $y = 0$ and $y = 2\delta$, respectively. Periodic boundary conditions are applied to both the spanwise and streamwise directions, and the no-slip boundary condition in the streamwise and spanwise directions is applied at the walls. A no-penetration boundary condition is used at the walls for the uncontrolled case, whereas blowing and suction boundary conditions are used for controlled cases. The implementation of this solver comes from previous studies~\cite{bae2021nonlinear}. From an RL perspective, the solver is considered as an environment.

\subsubsection{Control setups} 

The active control is achieved by applying a wall-normal velocity at the wall, which can either be blow or suction at walls. We call such a velocity a control or an action, denoted by $\phi$. Controls are applied at both walls (as shown on the left of~\Cref{fig-scenario}), while we focus on one wall in methodology formulation for simplicity. We use variables with subscript $w$ to denote physical variables associated with the wall. For example, we use $p_w$ to denote pressure at the wall. We interchangeably use the terms ``control'', ``action'', or ``actuation'' to denote $\phi$ in this paper. During the control, the channel flow is driven by a uniform mass \zelinrevision{flux}~\cite{OppositionControl, MLoppositionControl, DNSOptimalBenchmark}, which fixes the bulk Reynolds number $\re_b = u_b\delta/\nu$, where $u_b$ is the mass flow rate. This is achieved by adapting the mean pressure gradient ${\mathrm{d} P}/{\mathrm{d} x_1}$ at each time step~\cite{LargeEddy}.
The control target is to reduce the mean pressure gradient to achieve a drag reduction (DR):
\begin{equation}
DR = \frac{(-\frac{\mathrm{d} P}{\mathrm{d} x_1})_{t=0} - (-\frac{\mathrm{d} P}{\mathrm{d} x_1})_{t=T}}{(-\frac{\mathrm{d} P}{\mathrm{d} x_1})_{t=0}},
\label{eq-ns}
\end{equation}
where $(-{\mathrm{d} P}/{\mathrm{d} x_1})_{t=0}$ denotes the pressure gradient of the uncontrolled flow and $(-{\mathrm{d} P}/{\mathrm{d} x_1})_{t=T}$ denotes the pressure gradient after the control (at the termination timestep $t=T$).

\subsubsection{Computational domains} 

We discretize the computational domain via a staggered central finite-difference method. We consider two different simulation domains~\cite{DeepRLchannel}, where the first one is called a \textit{minimal channel} with size $\Omega=L_x \times L_y \times L_z = 1.77\delta \times 2\delta \times 0.89 \delta$. This channel flow is large enough to reflect relevant near-wall turbulent statistics and is used to reduce the computational burden. We also adopt another larger channel domain called a \textit{full channel}, which is of size $L_x=2\pi \delta$, $L_y=2\delta$, and $L_z=\pi \delta$. For the case with $\re_\tau \approx 180$, we discretize the streamwise and spanwise directions uniformly using $N_x \times N_z = 32 \times 32$ for the minimum channel and $N_x\times N_z = 128\times 128$ for the full channel, which result in streamwise and spanwise grid spacing of $\Delta x^+ \approx 10$ and $\Delta z^+ \approx 5$. Here, the superscript $+$ denotes the wall units defined by $\nu$ and $u_{\tau_0}$. In the wall-normal direction, we use a hyperbolic-tangent stretching function of size $N_y = 130$, which results in a wall-normal spacing of $\min(\Delta y^+) \approx 0.17$ at the wall and $\max(\Delta y^+) \approx 7.6$ at the center of the channel domain.
\begin{center}
\begin{table}
\def~{\hphantom{0}}
\begin{tabular}{@{}c|ccc|ccc@{}}
\toprule
        & \multicolumn{3}{c|}{Minimum channel}                      & \multicolumn{3}{c}{Full channel}                          \\ \midrule
$\re_b$ & \multicolumn{1}{c|}{Nx}  & \multicolumn{1}{c|}{Ny}  & Nz  & \multicolumn{1}{c|}{Nx}  & \multicolumn{1}{c|}{Ny}  & Nz  \\ \midrule
$3k$    & \multicolumn{1}{c|}{$32$}  & \multicolumn{1}{c|}{$130$} & $32$  & \multicolumn{1}{c|}{$128$} & \multicolumn{1}{c|}{$130$} & $128$ \\
$6k$    & \multicolumn{1}{c|}{$64$}  & \multicolumn{1}{c|}{$260$} & $64$  & \multicolumn{1}{c|}{$256$} & \multicolumn{1}{c|}{$260$} & $256$ \\
$9k$    & \multicolumn{1}{c|}{$96$}  & \multicolumn{1}{c|}{$390$} & $96$  & \multicolumn{1}{c|}{$384$} & \multicolumn{1}{c|}{$390$} & $384$ \\
$12k$   & \multicolumn{1}{c|}{$160$} & \multicolumn{1}{c|}{$520$} & $160$ & \multicolumn{1}{c|}{$512$} & \multicolumn{1}{c|}{$520$} & $512$ \\
$15k$   & \multicolumn{1}{c|}{$192$} & \multicolumn{1}{c|}{$650$} & $192$ & \multicolumn{1}{c|}{$640$} & \multicolumn{1}{c|}{$650$} & $640$ \\ \bottomrule
\end{tabular}
  \caption{A list of bulk Reynolds numbers $\re_b$ along with corresponding grid resolutions used in our study.}
  \label{tab-list-of-reynolds}
\end{table}
\end{center}

\subsubsection{The Reynolds numbers} 

For the remainder of the paper, we use $\re$ to denote $\re_b$. We experiment on different Reynolds numbers in the numerical study to test the generalization ability of the concerned methods. When changing the Reynolds numbers, we also proportionally scale the grid resolution $N_x$, $N_y$, and $N_z$ to preserve the grid resolutions in wall units. We give specific configurations of those parameters in~\Cref{tab-list-of-reynolds}.

\begin{figure}[H]
    \centering
    \includegraphics[width=\textwidth]{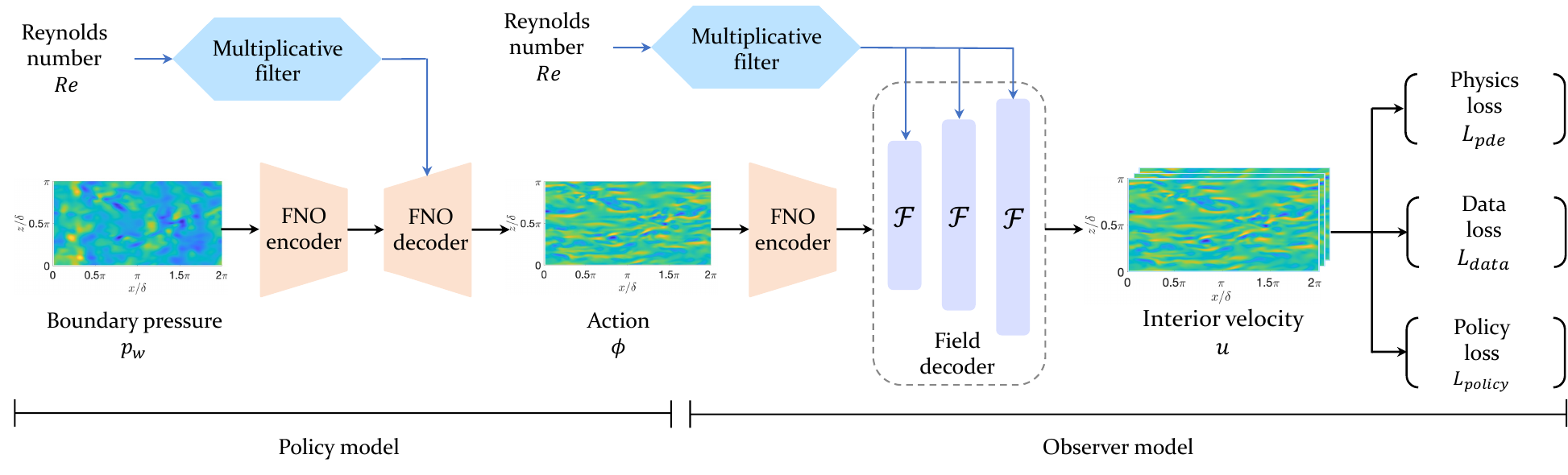}
    \caption{The policy model and the observer model in~\alg. Both the policy model and the observer model are conditioned on the Reynolds number $\re$, encoded by multiplicative filters~\citep{MFN}. The function of the policy model is to give the boundary velocity $\phi$ based on the boundary pressure $p_w$. It is instantiated with an \FNO encoder and an \FNO decoder. The output boundary velocity $\phi$ is sent to the observer model, which leverages the \FNO encoder to encode the boundary velocity. Then, the observer model uses a field decoder to output interior velocity $u$. We use a data loss $L_{data}$ (\Cref{eq-data-loss}) and a physics loss $L_{pde}$ (\Cref{eq-pde-loss}) to optimize the observer. Meanwhile, we adopt a policy loss $L_{policy}$ (\Cref{eq-policy-loss}) to optimize the policy model.}
    \label{fig-all-model}
\end{figure}

\subsection{Algorithm outline of \alg}
\label{sec-algorithm-outline}
In this paper, we propose a machine-learning-based framework for flow control inspired by previous predictive control (PC) studies~\citep{lee1998suboptimal,DNSOptimalBenchmark}. The overall schematic is shown in~\Cref{fig-all-model}. The proposed framework adopts a policy model denoted as~\policyname, and an observer model denoted as~\observername. The policy model~\policyname\space predicts the control $\phi$ based on the boundary pressure $p_w$. \zelinrevision{In our framework, we take the current pressure as input observation to respond to the current pressure and generate corresponding controls. Given the fact that we control the dynamics, the pressure is not uncontrolled time sequence, its dynamics is effected by control inputs. The design of controlling based on pressure is demonstrated effective in MP-CNN~\cite{MLoppositionControl}.} The observer model~\observername\space predicts the interior velocity field $u$ based on the control $\phi$. The velocity field $u$ is used to predict the outcome (the reward) of the control $\phi$ to revise the control accordingly. For convenience, we focus on the top wall in the whole section, while the control also happens in the bottom wall in practice.

We propose a machine learning algorithm to jointly optimize the policy model and the observer model in~\Cref{algorithm}. We use a memory (a ``replay buffer''~\cite{dqn}) to store the most recent collected data, which is used to train and optimize those models. Initially, the policy model and the observer model are initialized with zero weights, and the memory is initialized as an empty collection. The main loop of the algorithm iterates through episodes and conducts data collection, control, and learning during each episode: At the start of one episode (Line $4$), \alg collects the wall pressure $p_w$, the field velocity $u$ and the drag $-\frac{dP}{dx_1}$ from the solver, and stores them to the memory. \zelinrevision{The collected drag $-\frac{dP}{dx_1}$ is used in computing the physics loss $L_{pde}$, which will be discussed in detail,~\Cref{subsec-pde-loss}.} Next, the policy model predicts the control $\phi$ based on the boundary pressure (Line 5). Then, the control $\phi$ is applied to the wall, and the solver is updated to the next timestep (Line $6$). Subsequently, the algorithm trains and optimizes the observer model and the policy model based on newly collected data (Lines $8-12$). For each training epoch, a data tuple is sampled from memory (Line $9$). Then, the observer model is learned through optimizing two loss functions (introduced later): supervised loss $L_{data}$ (\Cref{eq-data-loss}) and physics loss $L_{pde}$ (\Cref{eq-pde-loss}). After that, the policy model is trained to optimize a policy loss $L_{policy}$ (Line $11$). Here, the observer model is used to predict the outcome of the policy model, so when training the policy model, the observer model is fixed, which means the weights of the observer model are not changed during the optimization of the policy model. \zelinrevision{Note that our observer model is not fixed throughout all episodes of training the policy model, and hence, it incorporates different dynamics. Specifically, Line $10$ in~\Cref{algorithm} means the observer can learn from the collected experiences of different controls, because the memory is retained from prior episodes in the replay buffer} 

\zelinrevision{We build the policy model and the observer model based on neural operators~\cite{neuralOperator}. Neural operators are neural networks that learn the map between infinite-dimensional function spaces, such as the Fourier neural operator (\FNO~\cite{FNO}). Since the states in the dynamics include pressure, which is a function, neural operators are particularly well-suited for modeling turbulent flows. Their ability to output functions allows us to incorporate physics-informed losses during training, further enhancing their effectiveness. Prior studies~\cite{neuralOperator} have shown that neural operators outperform existing machine-learning-based methodologies, and we verify in our experiments that our neural-operator-based model achieves superior performance compared to prior state-of-the-art approaches. Our models contain approximately 0.35 million parameters in total, striking a balance between expressive capacity and computational efficiency.}

\begin{algorithm}[]
\caption{Physics-informed neural-operator-based predictive control (\alg)}
\begin{algorithmic}[1]
\State Initialize the policy model, the observer model, and the memory;
\For{each episode}
    \For{each timestep in episode} \Comment{Data collection: roll out and collect trajectory}
        \State Collect $p_w$, $\boldsymbol{u}$, and $-\frac{dP}{dx_1}$ to the memory;
        \State Predict the control $\phi$ with the policy model \policyname;
        \State Apply the control $\phi$ to the wall, and step the solver to the next timestep;
    \EndFor
    \For{each epoch} 
        \State Sample a data tuple ($p_w$, $u$, $-\frac{dP}{dx_1}$) from the memory;
        \State Update the observer model based on a combined loss $L_{data} + L_{pde}$;
        \State Fix the observer model, update the policy model to optimize the target $L_{policy}$;
    \EndFor
\EndFor
\end{algorithmic}
\label{algorithm}
\end{algorithm}

\subsection{The policy model \texorpdfstring{\policyname}{MP}}
\label{sec-policy-model}

We introduce a policy model \policyname\space to predict the control $\phi$ based on the boundary pressure $p_w$. The structure of the policy model is shown on the left of~\Cref{fig-all-model}. 

\subsubsection{\textbf{The~\FNO encoder of the policy model}} 

First, the policy model leverages an \FNO~\citep{FNO}, which is a variant of neural operators~\cite{neuralOperator}. Given the boundary pressure $p_w:\Gamma^{+}\rightarrow \mathbb{R}$, we use an \FNO encoder to encode it to the latent function $h_{p}$:
\begin{equation}
    h_p = \FNO(p_w).
\end{equation}

\subsubsection{\textbf{Conditioning on the Reynolds number}} 

To help the model adapt to unseen Reynolds numbers, we also condition the decoder model on the Reynolds number $\re$. We propose to use the multiplicative filter (\MFN)~\citep{MFN} to encode the Reynolds number $\re$ to get latent representations. The multiplicative filter takes the hidden feature and the Reynolds number $\re$ as input and outputs a new feature $h_{pm}$:
\begin{equation}
    h_{pm} = \MFN(h_p),
\end{equation}
where the \MFN is based on the sinusoidal filter $g$ to encode the Reynolds number $\re$:
\begin{equation}
g\left(\re; \theta^{(i)}\right)=\sin \left(\omega^{(i)} \frac{\re}{\operatorname{Re_m}}+\tau^{(i)}\right).
\end{equation}
Here $\theta^{(i)}=\{\omega^{(i)}, \tau^{(i)}\}$ are parameters of the sinusoidal filter, and a constant Reynolds number $\operatorname{Re_m}=100,000$ is used to normalize the input Reynolds number $\re$. The \MFN then performs the following recursion with $L$ layers:
\begin{equation}
\begin{aligned}
z^{(1)} & =h_p, \\
z^{(i+1)} & =\left(W^{(i)} z^{(i)}+b^{(i)}\right) \circ g\left(\re ; \theta^{(i+1)}\right), i=1, \ldots, L-1, \\
h_{pm} & =W^{(L)} z^{(L)}+b^{(L)}, \\
\end{aligned}
\end{equation}
where $W^{(i)}$ and $b^{(i)}$ are the learnable linear transform and bias. 

\subsubsection{\textbf{The \FNO decoder of the policy model}} 

Finally, the policy model uses an \FNO decoder to get the control:
\begin{equation}
    \phi = M_p(p_w) = \FNO(h_{pm}) - \operatorname{mean}(\FNO(h_{pm})),
\label{eq-policy-model}
\end{equation}
where we use a normalization function to ensure that we don't add mass to the system.

\subsubsection{\textbf{The policy loss}} 

To optimize the policy model, we adopt a policy loss (corresponding to Line 11 in~\Cref{algorithm}). The policy loss is written in two terms: the turbulent kinetic energy (TKE) and the norm of the control:
\begin{equation}
L_{policy}(\phi) = E_t\left(\int_{\Omega}|\boldsymbol{u}(t+\Delta t)|^2 \mathrm{~d} \boldsymbol{x}+\frac{\lambda_{n}}{\Delta t} \int_{\Gamma_2^+} \int_t^{t+\Delta t} \phi^2 \mathrm{~d}\tau\mathrm{~d}S \right).
\label{eq-policy-loss}
\end{equation}
In this equation, the expectation is taken over the concerned episode. The TKE is computed based on the field velocity, where the field velocity is the predicted outcome of the control with the fixed observer model (this prediction procedure will be introduced in the next subsection). We use $\boldsymbol{u}(t+\Delta t)$ to denote the velocity at the time after a period $t+\Delta t$, which helps obtain long-term gain. 
Here $\lambda_n = 0.5$ is a balancing term of the regularization term, and $\Delta t$ is the concerned time window.

\zelinrevision{The kinetic energy, whose minimization is associated with the suppression of turbulence, subsequently results in drag reduction. In this work, we utilize this association. Usually, reducing the turbulent kinetic energy causes a decrease in drag, where more background of this can be found in~\cite{DNSOptimalBenchmark}.}

\subsection{The observer model}
\label{sec-observer-model}
In this subsection, we present a PDE observer model called \observername, which predicts internal velocity field~$\boldsymbol{u}$ given the control $\phi$. We use $\phi_t$ to denote the control at discrete timestep $t$. The observer model is also conditioned on the Reynolds number $\re$ to boost generalization to different Reynolds numbers. The observer model is shown on the right of~\Cref{fig-all-model}.

\subsubsection{\textbf{The \FNO encoder of the observer model}} 

The boundary velocity $\{\phi\}$ is normalized and passed through an \FNO encoder. The \FNO encoder takes the boundary velocity as input and outputs a hidden feature of controls $h_c: \Gamma^{+} \rightarrow \mathbb{R}^{d_1}$ where $d_1$ is the hidden feature dimension.

\subsubsection{\textbf{The field decoder of the observer model}} 

We then use a field decoder to transform the hidden feature $h_c$ to the field velocity $\boldsymbol{u}: \Omega \rightarrow \mathbb{R}$, $\boldsymbol{v}: \Omega \rightarrow \mathbb{R}$, and $\boldsymbol{w}:\Omega \rightarrow \mathbb{R}$. To achieve this, we first generate latent representations for each of the field velocities, which is performed by an inflating hidden function $h_{in}: \Omega \rightarrow \mathbb{R}^{d_2}$, ($d_2$ is the dimension of the inflated feature) turning the 2D hidden feature $h_c$ to the 3D space:
\begin{equation}
h_{in}(x, y, z) = h_c(x, z) \oplus \operatorname{PosEmb}(y),
\end{equation}
where $\oplus$ denotes the concatenation operator, and $\operatorname{PosEmb}$ is a positional embedding function which turns $y$ into a hidden feature:
\begin{subequations}
\begin{equation}
\gamma(y, j) =
    \begin{cases}
        \sin(2^{\lfloor j / 2 \rfloor}\pi y), & \text{if } j \bmod 2 = 0,\\
        \cos(2^{\lfloor j / 2 \rfloor}\pi y), & \text{else.}
    \end{cases}
\label{eq-fourier-features-1}
\end{equation}
\begin{equation}
\operatorname{PosEmb}(y) = \gamma(y, 1)  \oplus \gamma(y, 2)  \oplus \ldots  \oplus \gamma(y, n_p),
\label{eq-fourier-features-2}
\end{equation}
\label{eq-fourier-features}
\end{subequations}
where $n_p$ is the number of the trigonometric functions adopted to form the positional embeddings. After that, we decode the hidden functions into $u, v, w$ with 3D \FNO modules~\cite{FNO}:
\begin{equation}
u = \operatorname{FNO3D}(\operatorname{h_{in}}), v = \operatorname{FNO3D}(\operatorname{h_{in}}), w = \operatorname{FNO3D}(\operatorname{h_{in}}).
\end{equation}

\paragraph{Discussions on the encoder-decoder structure}
\zelinrevision{
The encoder-decoder structure in our policy model is not primarily for data compression as in the conventional sense but rather for learning an efficient and structured representation of the input flow state. While it is true that the dimensionality of the input $p_w$ and output $\phi$ is consistent, a direct mapping from $p_w$ to
$\phi$ without an encoder-decoder structure may not fully capture the complex and nonlinear relationships necessary for effective turbulence control~\citep{FNO}.}

\zelinrevision{
As it is a common practice in deep learning to develop efficient models, the encoder serves to extract meaningful, low-dimensional latent features that are most relevant to control decisions, effectively filtering out irrelevant information or noise. This is especially crucial in high-dimensional flow fields where direct mappings can be overly sensitive to variations and may not generalize well. The decoder then reconstructs the optimal control action based on these learned features, ensuring robustness and stability.}

\subsubsection{\textbf{The data loss}} 

We introduce the data loss to train the observer model, which penalizes the $L2$ distance between the predicted and the ground-true field velocity, denoted by $u$ and $u_{gt}$ correspondingly:
\begin{equation}
L_{data}=\mathbb{E}_{x, y, z}\left(\frac{u_{gt}(x, y, z)-u(x, y, z)}{\bar{u}(x, y, z)}\right)^2,
\label{eq-data-loss}
\end{equation}
where $\bar{u}$ is the root-mean-square velocity at each point $(x, y, z)$, and the expectation is taken in the full domain $\Omega$.

\subsubsection{\textbf{Physics-informed learning and the PDE loss}}
\label{subsec-pde-loss}

Despite the supervised loss introduced in~\Cref{eq-data-loss}, we further introduce a PDE loss $L_{pde}$ to optimize the model by leveraging the governing PDE (\Cref{eq-pde}). This approach belongs to physics-informed learning~\citep{PINN, PINO}, where through optimizing the PDE loss, the observer can be optimized without ground-true data acquired from precise simulation. This technique is helpful when the training data is scarce, such as in the high Reynolds number region. Our experiments show that physics-informed learning can help the observer model generalize to unseen Reynolds numbers.

We implement the PDE loss based on the difference between the temporal gradient of the predicted velocity and the right-hand-side terms. We denote the predicted velocity via $\frac{du_\theta}{dt}$, $\frac{dv_\theta}{dt}$, and $\frac{dw_\theta}{dt}$, and we denote rest terms via a function $R$, then the PDE loss is given as:
\begin{equation}
L_{pde} = \left|\frac{du_\theta}{dt} - R(u_\theta) \right| + \left|\frac{dv_\theta}{dt} - R(v_\theta) \right| + \left|\frac{dw_\theta}{dt} - R(w_\theta) \right|,
\label{eq-pde-loss}
\end{equation}
where the velocity gradients are estimated via temporal difference, and we compute $R(u_i)$ as:
\begin{equation}
R(u_i) = -u_j \frac{\partial u_i}{\partial x_j} -\frac{\mathrm{d} P}{\mathrm{d} x_1} \delta_{1 i} -\frac{\partial p}{\partial x_i}+\frac{1}{\re_\tau} \frac{\partial^2 u_i}{\partial x_j \partial x_j}.
\end{equation}
Here, $u_i$ denotes any of the predicted velocity $u_\theta$, $v_\theta$, and $w_\theta$.





\section{Discussion}
\label{sec-summary}
In this work, we address the challenge of turbulent flows in the wall-bounded scenario. We consider an active control setup, implementing control through blowing and suction at the wall. We propose a novel machine-learning-based predictive control scheme, \alg. This framework leverages a policy model and an observer model. The policy model is used to predict the control (applied boundary velocity) based on the boundary pressure. The observer model predicts the control outcome (internal field velocity) based on the control.

\zelinrevision{We test our method on Reynolds numbers that are unseen during training, specifically, higher Reynolds numbers that correspond to highly turbulent flows. Fluid flows with different $\re$ have shared features at multiple scales. Even then, adapting the control to unseen $\re$, especially higher $\re$ is challenging due to increased nonlinear interactions. Our method works effectively even under this challenging setting since it can adapt online to unseen scenarios while also utilizing the shared features from its earlier training. Such transfer learning across different $\re$ can be further enhanced by explicitly incorporating relationships across different scales, which is of interest for further investigation.}

Our approach demonstrates superior accuracy and drag reduction compared to alternative machine-learning methods. Notably, \alg achieves a remarkable 43.5\% drag reduction for Reynolds numbers not included in the training data, surpassing both opposition control and the optimal control baseline. The proposed iterative learning procedure, with extensive observer and policy learning, proves effective in achieving more robust turbulence control. This work provides a foundation for more efficient and practical turbulence control methodologies.

\zelinrevision{Our \alg is not just data-driven but is physics-informed and, hence, can generalize beyond a training regime. However, if the domain shift is drastic, e.g., a very high $\re$, we do not expect our method to succeed, and it is an open question if further algorithmic development is possible.}

For the present problem setup, we test our method on unseen Reynolds numbers, specifically, highly turbulent flows with high Reynolds numbers, for which we observe positive results. In the case of high Reynolds numbers, extending our model learning, along with control and policy learning, is challenging in the plain setting of function-to-function map learning. However, due to the symmetry in PDEs, behavior at different  Reynolds numbers share the same physics, and stronger algorithmic developments are needed to utilize this characteristic of our problem setup, which counts as a limitation of our current method.

Furthermore, in the empirical study, we consider fixed placement of the sensors. The proposed method method applies to any sensor configurations, which is a direct result of characterizing the problem formulation in function space. A study on the effect of the sensory configuration on the final performance of the algorithm is of interest.


\clearpage
\newpage 

\section{Appendix}

\label{sec-appendix}
\subsection{Comparisons with other methods}
A comparative analysis of our approach and related methodologies is provided in~\Cref{tab-method-comparisons}, affording a comprehensive understanding of \alg's features.

\begin{table*}[t!]
\def~{\hphantom{0}}
\resizebox{\columnwidth}{!}{%
\begin{tabular}{c|c|c|c|c|c|c}
\toprule
Method name               & Opposition                & \DNSPC & Local suboptimal          & MP-CNN                       & DDPG                     & \alg       \\ \midrule
Reference                        & \citep{OppositionControl} & \citep{DNSOptimalBenchmark}  & \citep{lee1998suboptimal} & ~\citep{MLoppositionControl} & ~\citep{DeepRLchannel}   & Ours          \\ \midrule \midrule
1. Machine learning model backbone       & N/A                       & N/A                          & N/A                       & CNN                          & FCN & Neural operators \\ \midrule
2. Based on predictive control          & \budui                    & \dui                         & \dui                      & \budui                       & \budui                   & \dui       \\ \midrule
3. Only need boundary observation    & \budui                    & \budui                       & \dui                      & \dui                         & \dui                     & \dui       \\ \midrule
4. Experiment with varied Reynolds numbers & \budui                      & \budui                       & \budui                    & \dui                         & \budui                   & \dui       \\ \midrule
5. Use a PDE observer           & N/A                       & N/A                          & N/A                       & \dui                       & \budui                   & \dui       \\ \midrule
6. Use physics-informed learning    & N/A                       & N/A                          & N/A                       & \budui                       & \budui                   & \dui       \\ \bottomrule
\end{tabular}
}
  \caption{In the context of the flow control problem, we present comparisons between previous and our approaches. Each column corresponds to a specific method, and each row denotes a particular property. The first property indicates the machine learning model used in those approaches, with the first three methods not employing machine learning techniques. The second property pertains to whether a method is grounded in predictive control. \alg is rooted in predictive control as it performs control based on the predicted impact of the boundary velocity. The subsequent property addresses whether the method can function solely with boundary information, excluding the need for internal field data. All three machine learning models listed can achieve this, except for Local suboptimal~\cite{lee1998suboptimal}. The fourth property outlines whether a method has been verified in flows of varied Reynolds numbers. The final two properties are techniques used in machine learning methods. Both MP-CNN~\cite{MLoppositionControl} and our approach leverage a PDE observer. Furthermore, our model is the only approach that employs physics-informed learning in the control procedure.}
  \label{tab-method-comparisons}
\end{table*} 

\subsection{Theorectical results of~\DDPG}

\DDPG models the control problem as a Markov decision process (MDP) in function space. The MDP consists of several components: a state space $\mathcal{X}$, an action space $\mathcal{A}$. A state $x\in \mathcal{X}$ and an action $a\in \mathcal{A}$ are functions. We use $p_1(x)$ to denote the initial probabilistic measure over states, and $p(x_{t+1}|x_{t})$ to denote the probabilistic measure describing the transition dynamics distribution, while the MDP satisfies the Markov property $p(x_{t+1}|x_1, ..., x_t, a_t)=p(x_{t+1}|x_t, a_t)$, for any trajectory $x_1, a_1, x_2, a_2, ..., x_T, a_T$ in state-action space, and a reward function $r: \mathcal{X}\times \mathcal{A} \rightarrow \mathbb{R}$. We define a deterministic policy $\mu_\theta: \mathcal{X} \rightarrow \mathcal{A}$ parameterized by $\theta$. The discount factor $\gamma \in [0, 1)$ is given to calculate the total discounted reward (i.e., the return) $r_t^\gamma$. We denote the Q function to be $Q^\mu (x, a)=\mathbb{E}[r_1^\gamma | X_1=x, A_1=a;\mu]$, and we denote the V function as $V^\mu(x) = \mathbb{E}[r_1^\gamma | X_1=a;\mu]$. For simplicity, we superscript value functions by $\mu$ instead of $\mu_\theta$. We denote the Fréchet derivative of the state-action value function with respect to the action to be $D_a(Q^\mu(x, a))$, and we denote the derivative of the transition probability to be $D_a(p(x_{t+1}|x_{t}, a))$. We further denote the discounted state occupancy measure by $\rho^\mu(x')=\int_{\mathcal{X}} \sum_{t=1}^{\infty} \gamma^{t-1} p_1(x) p\left(x \rightarrow x^{\prime}, t, \mu\right) \mathrm{d} x$, and we consider the following performance objective:
\begin{equation}
J(\mu_\theta) = \mathbb E_{x\sim \rho^\mu}[r(x, \mu_\theta(x)] = \int_{\mathcal{X}} V^{\mu}(x)\mathrm{d} p_1(x).
\label{eq-target-j}
\end{equation}

Then, we introduce the deterministic policy gradient theorem for neural operators as follows:

\begin{theorem}[Deterministic policy gradient theorem for neural operators]
Suppose that the MDP satisfies the following regularization conditions 
\begin{enumerate}
    \item $p(x'|x, a)$, $D_a(Q^\mu(x, a))$, $\mu_\theta(x)$, $r(x,a)$, $D_a r(x, a)$, and $p_1(x)$ are continuous in all parameters and variables $x, a, x'$ and $\theta$,
    \item there exists $b$ and $L$ such that $\sup_s p_1(x) < b$, $\sup_{a, x, x'} p(x'|x, a) < b$, $\sup_{a,x} \|r(x, a)\| < b$, $\sup_{a, x, x'}\| D_a p(x'|x, a)\|<L$, and $\sup_{a,x} \| D_a r(x, a) \| < L$,
\end{enumerate}
then $\nabla_\theta \mu_\theta(x)$ and $D_a(Q^\mu(x, a))$ exist and the deterministic policy gradient is given as:
$$
\begin{aligned}
\nabla_\theta J\left(\mu_\theta\right) & =\left.\int_{\mathcal{X}} \rho^\mu(x) \nabla_\theta \mu_\theta(x) D_a(Q^\mu(x, a))\right|_{a=\mu_\theta(x)} \mathrm{d} x \\
& =\mathbb{E}_{x \sim \rho^\mu}\left[\left.\nabla_\theta \mu_\theta(x) D_a(Q^\mu(x, a))\right|_{a=\mu_\theta(x)}\right].
\end{aligned}
$$
\end{theorem}

\begin{proof}
The proof follows the deterministic policy gradient algorithms~\citep{DDPG}. We first derive the gradient of the value function as:
\begin{align}
    \nabla_\theta V^{\mu}(x) &= \nabla_\theta Q^{\mu}\left(x, \mu_\theta(x)\right) \label{proof1-line1} \\
     &= \nabla_\theta\left(r\left(x, \mu_\theta(x)\right)+\int_{\mathcal{X}} \gamma p\left(x^{\prime} \mid x, \mu_\theta(x)\right) V^{\mu}\left(x^{\prime}\right) \mathrm{d} x^{\prime}\right), \label{proof1-line2} \\
    &= \left.\nabla_\theta \mu_\theta(x) D_a r(x, a)\right|_{a=\mu_\theta(x)} + \notag\\
    & +\int_{\mathcal{X}} \gamma\left(p\left(x^{\prime} \mid x, \mu_\theta(x)\right) \nabla_\theta V^{\mu}\left(x'\right)+\left.\nabla_\theta \mu_\theta(x) D_a p\left(x' \mid x, a\right)\right|_{a=\mu_\theta(x)} V^{\mu}\left(x'\right)\right) \mathrm{d} x' \label{proof1-line3} \\
    &= \left.\nabla_\theta \mu_\theta(x) D_a\left(r(x, a)+\int_{\mathcal{X}} \gamma p\left(x' \mid x, a\right) V^{\mu}\left(x'\right) \mathrm{d} x'\right)\right|_{a=\mu_\theta(x)} \notag\\
    &+ \int_{\mathcal{X}} \gamma p\left(x^{\prime} \mid x, \mu_\theta(x)\right) \nabla_\theta V^{\mu}\left(x'\right) \mathrm{d} x'\label{proof1-line4} \\
    &= \left.\nabla_\theta \mu_\theta(x) D_a Q^{\mu}(x, a)\right|_{a=\mu_\theta(x)}+\int_{\mathcal{X}} \gamma p\left(x \rightarrow x', 1, \mu_\theta\right) \nabla_\theta V^{\mu}\left(x'\right) \mathrm{d} x'\label{proof1-line5} \\
    &= \left.\int_{\mathcal{X}} \sum_{t=0}^{\infty} \gamma^t p\left(x \rightarrow x', t, \mu_\theta\right) \nabla_\theta \mu_\theta\left(x'\right) D_a Q^{\mu}\left(x', a\right)\right|_{a=\mu_\theta\left(x'\right)} \mathrm{d} x'\label{proof1-line6},
\end{align}
where we apply the definition of the value function in~\Cref{proof1-line1} and~\Cref{proof1-line2}. We apply the chain rule for Fréchet derivative in~\Cref{proof1-line3}. We use the Leibniz integral rule to exchange the order of derivative and integration based on the regularization conditions in~\Cref{proof1-line5} and~\Cref{proof1-line6}. Furthermore, \Cref{proof1-line6} is derived by iterating \Cref{proof1-line5} based on the formula of the value derivative $\nabla_\theta V^{\mu}$.

Now consider the definition of the target in~\Cref{eq-target-j}, we can derive that:
\begin{align}
\nabla_\theta J\left(\mu_\theta\right) & =\nabla_\theta \int_{\mathcal{X}} p_1(x) V^{\mu}(x) dx\label{proof1-line7} \\
& =\int_{\mathcal{X}} p_1(x) \nabla_\theta V^{\mu}(x) dx \label{proof1-line8}\\
& =\left.\int_{\mathcal{X}} \int_{\mathcal{X}} \sum_{t=0}^{\infty} \gamma^t p_1(x) p\left(x \rightarrow x', t, \mu_\theta\right) \nabla_\theta \mu_\theta\left(x'\right) D_a Q^{\mu}\left(x', a\right)\right|_{a=\mu_\theta\left(x'\right)} dx^{\prime} dx \label{proof1-line9}\\
& =\left.\int_{\mathcal{X}} \rho^{\mu}(x) \nabla_\theta \mu_\theta(x) D_a Q^{\mu}(x, a)\right|_{a=\mu_\theta(x)} dx,\label{proof1-line10}
\end{align}
where we leverage again exchange the order of integration and derivative in~\Cref{proof1-line9} and~\Cref{proof1-line10}, and we consider the definition of $\rho^{\mu}$ in the last line.
\end{proof}


\begin{table*}
\resizebox{\columnwidth}{!}{
\def~{\hphantom{0}}
\begin{tabular}{@{}c|cc|cc|cc|cc@{}}
\toprule
$\re_b$                                              & \multicolumn{1}{c|}{3k}                & 3k                & \multicolumn{1}{c|}{3k, 6k, 9k, 15k}   & 12k               & \multicolumn{1}{c|}{3k, 6k, 9k, 12k}   & 15k               & \multicolumn{1}{c|}{3k}                & 6k, 9k, 12k, 15k  \\ \midrule
Phase                                                & \multicolumn{1}{c|}{Training}          & Testing           & \multicolumn{1}{c|}{Training}          & Testing           & \multicolumn{1}{c|}{Training}          & Testing           & \multicolumn{1}{c|}{Training}          & Testing           \\ \midrule
Ours (w/o PC)                                        & \multicolumn{1}{c|}{34.3}          & 33.5          & \multicolumn{1}{c|}{36.0}          & 31.1          & \multicolumn{1}{c|}{38.0}          & 32.3          & \multicolumn{1}{c|}{34.3}          & 15.3          \\ \midrule
Ours (FNO~\cite{FNO} $\rightarrow$ CNN~\citep{MLoppositionControl})               & \multicolumn{1}{c|}{39.8}          & 36.5         & \multicolumn{1}{c|}{40.7}          & 37.9          & \multicolumn{1}{c|}{34.9}          & 31.6         & \multicolumn{1}{c|}{41.8}          & 36.9          \\ \midrule
Ours (FNO~\citep{FNO}$\rightarrow$ RNO~\citep{RNO} ) & \multicolumn{1}{c|}{43.4}          & 42.0          & \multicolumn{1}{c|}{42.9}          & 39.9         & \multicolumn{1}{c|}{39.4}          & 33.8          & \multicolumn{1}{c|}{43.4}          & 40.8        \\ \midrule
Ours (w/o MF~\cite{MFN})                                        & \multicolumn{1}{c|}{43.1}          & 42.0          & \multicolumn{1}{c|}{43.0}          & 39.8         & \multicolumn{1}{c|}{35.8}          & 32.2         & \multicolumn{1}{c|}{43.1}          & 37.1          \\ \midrule
Ours (w/o $L_{pde}$)                                        & \multicolumn{1}{c|}{42.0}          & 39.2          & \multicolumn{1}{c|}{41.8}          & 38.4 & \multicolumn{1}{c|}{36.0}          & 32.9          & \multicolumn{1}{c|}{42.0}          & 38.8          \\ \midrule
\zelinrevision{Ours (pressure $\rightarrow$ shear stresses)}                                                & \multicolumn{1}{c|}{37.1} & 35.8 & \multicolumn{1}{c|}{36.9} & 35.9 & \multicolumn{1}{c|}{37.4} & 33.0 & \multicolumn{1}{c|}{38.8} & 37.7 \\ \midrule \midrule
\zelinrevision{Ours (w/ noise $\frac{1}{SNR}=0.05$)}                                                & \multicolumn{1}{c|}{41.0} & 41.9 & \multicolumn{1}{c|}{39.9} & 38.4 & \multicolumn{1}{c|}{39.6} & 34.9 & \multicolumn{1}{c|}{42.5} & 38.3 \\ \midrule
\zelinrevision{Ours (w/ noise $\frac{1}{SNR}=0.10$)}                                                & \multicolumn{1}{c|}{38.9} & 36.1 & \multicolumn{1}{c|}{36.5} & 35.0 & \multicolumn{1}{c|}{38.8} & 32.8 & \multicolumn{1}{c|}{40.8} & 35.9 \\ \midrule
\zelinrevision{Ours (w/ noise $\frac{1}{SNR}=0.20$)}                                                & \multicolumn{1}{c|}{35.0} & 33.9 & \multicolumn{1}{c|}{31.9} & 30.4 & \multicolumn{1}{c|}{37.6} & 31.9 & \multicolumn{1}{c|}{35.5} & 33.1 \\ \midrule
\midrule \midrule
Ours                                                 & \multicolumn{1}{c|}{\textbf{43.5}} & \textbf{42.1} & \multicolumn{1}{c|}{\textbf{43.1}} & \textbf{40.3} & \multicolumn{1}{c|}{\textbf{40.1}} & \textbf{35.1} & \multicolumn{1}{c|}{\textbf{43.5}} & \textbf{39.0} \\ \bottomrule
\end{tabular}}
\caption{\zelinrevision{(This table has been expanded to incorporate more ablation studies recommended by Reviewer 1.) Ablation studies of~\alg. The metric is the drag reduction rate in the full channel flow of varied Reynolds numbers.}}
\label{tab-ablation-study}
\end{table*}

\subsection{Analysis via the supervised representation learning} 

\begin{figure}
    \centering 
    \includegraphics[width=1.0\textwidth]{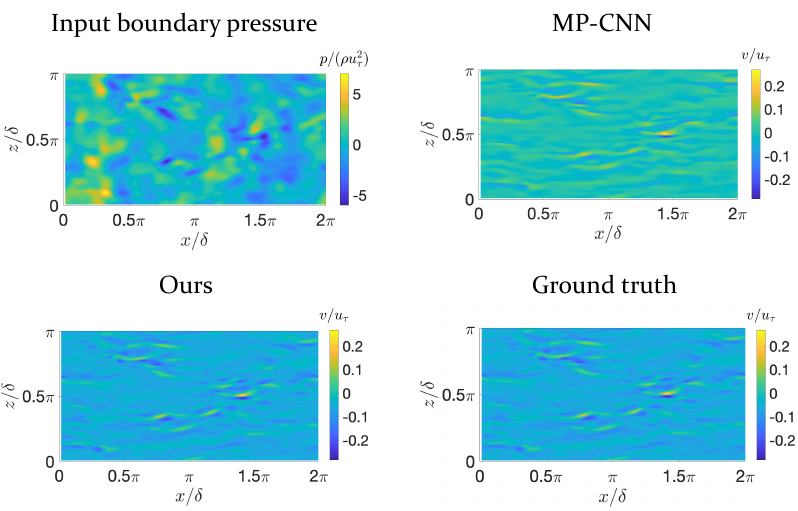}
    \caption{Visualization of input pressure and predicted interior velocity $v_{y^+=10}$ in the full channel scenario.}
    \label{fig-2d-visualization}
\end{figure}

In this subsection, we collect several datasets to train and evaluate observer models, while the target is to measure the fitting performance of each model. All collected datasets use the full channel flow. We change the Reynolds number by altering the kinematic viscosity, as stated in~\Cref{sec-problem-setting}. Those datasets are obtained from the DNS of a turbulent channel flow, with a bulk velocity Reynolds number of $3k$, $6k$, $9k$, $12k$, and $15k$, correspondingly. \zelinrevision{We experiment with four setups with varied Reynolds numbers.} Descriptions of those datasets are presented in the main text. Different from other experiments that do not use the normal velocity $v(y^+=10)$ as a supervision signal, in this section, all models are trained and compared under a supervised learning setup where the normal velocity is the ground truth. Each setup has flow data of three splits: \texttt{train}, \texttt{validation}, and \texttt{test}. The \texttt{training} split is used to optimize machine-learning models, the \texttt{validation} split is used to tune hyper-parameters of models, and the \texttt{test} split is for evaluating machine-learning models. We use $700$ instantaneous fields for the \texttt{training} split, $100$ fields for the \texttt{validation} dataset, and 100 different fields for the \texttt{test} split. If a split contains more than one Reynolds number, then the number of flow data is equal for each Reynolds number. The training dataset size is approximately the same as~\citet{MLoppositionControl} and is enough to train neural models. The collected wall pressure and normal velocity are normalized by their root-mean-square values before being fed into neural models.

We first compare the representation power of~\alg's observer model against another observer of \MPCNN~\cite{MLoppositionControl} under a supervised learning setting, where we call the observer model used in~\alg as \observername. In this supervised learning setup, all models take the wall pressure $p_w$ as input, predicting the normal velocity $v(y^+=10)$ at the detection layer $y^+=10$. Trained models under this supervised learning framework can be further applied to opposition control~\cite{MLoppositionControl}, where details will be given later in this subsection.

We compare to \MPCNN~\cite{MLoppositionControl}, which contains twenty hidden layers, an averaging pooling layer, and a linear layer with residual connections. Zero paddings are used to adjust the sizes of convolutional filters. Zero paddings are applied when the height or width in the input is an odd number. The input and output grid points are $32\times 32$ and $16\times 16$, corresponding to the $x$-axis and $z$-axis. The input and output are aligned in their centers. The weights in the model are initialized by the Xavier method~\cite{xavier}. We do not use the GAN loss~\cite{SRGAN} to train the \MPCNN~\cite{MLoppositionControl} because computing the GAN loss needs another CNN as a discriminator. We experiment with two loss setups to test the performance of \MPCNN~\cite{MLoppositionControl}. The first setup is to only train the model with $L_{data}$ (w/o physics-informed learning), and the other setup is to train the model with $L_{data}$ and $L_{pde}$ (w/ physics-informed learning). The model parameters are optimized with Adam~\cite{kingma2014adam} with an initial learning rate of $1e-3$. No learning rate scheduler is adopted.

In this case, we let the observer model only output the interior velocity in the detection plane $v(y^+=10)$. We use the same size of inputs and outputs as \MPCNN~\cite{MLoppositionControl}. The number of parameters in our observer model is smaller than that of \MPCNN. The optimizer and learning rate remains the same as that of \MPCNN.

\Cref{fig-scatter-plot} shows scatter plots of the prediction and ground truth data, where these diagrams are produced under the last generalization setup. Under this setup, the training dataset uses a Reynolds number of $3k$ (corresponding to the first row of this figure), while testing datasets use varied Reynolds numbers of $6k$ to $15k$. We experiment with two different models by changing training losses. The first setting is to train~\observername\space only with the data loss $L_{data}$ (denoted by ``Ours w/o $L_{pde}$''), while the physics-informed learning is not adopted in this case. The second setting is to train our model with the data loss $L_{data}$ and the PDE loss $L_{pde}$. This full model is denoted by ``Ours''. In this plot, the x-axis denotes the ground truth, and the y-axis denotes the predicted values, where the velocity tensor is flattened into $1D$ before plotting this diagram. Therefore, the distance of each scatter point to the $y=x$ line (which means ground truth equals the prediction) reflects the prediction error. We observe that \MPCNN~\cite{MLoppositionControl} performs well in a bulk-velocity Reynolds number $\re_b=3k$, but it deteriorates significantly in the high Reynolds number scenarios. \observername\space demonstrates superior performance than \MPCNN~\cite{MLoppositionControl}, even without physics-informed learning (the PDE loss $L_{pde}$). The PDE loss can enhance predictions of~\observername, especially when the flow is highly turbulence. This demonstrates that~\observername\space learns better neural features than \MPCNN~\cite{MLoppositionControl}.

\Cref{fig-2d-visualization} provides 2D visualizations of input pressure and output velocity comparing ours against the \MPCNN~\cite{MLoppositionControl} baseline. All models are trained only with $L_{data}$ under a Reynolds number of $3k$ and are tested in an unseen Reynolds number of $6k$. In this figure, predicted velocities are on the test split. Our methods can produce closer predictions to the ground truth, while \MPCNN fails to predict the target velocity in many regions.


\subsection{More ablation studies of~\alg}
\label{sec-model-ablations}
\zelinrevision{In this subsection, we compare various machine-learning-based models and, therefore, provide rationals and insights behind our model choice. The result of this ablation study is presented in~\Cref{tab-ablation-study}. The results demonstrate that each component of our approach contributes to its overall effectiveness, particularly in challenging generalization scenarios.}

\zelinrevision{First, we analyze the role of \textbf{predictive control (PC)} by removing them from the training process. Without PC, the model achieves significantly lower drag reduction, particularly in the challenging settings to higher Reynolds numbers, suggesting that explicit physical priors are crucial for robust generalization. Similarly, we replace FNO~\cite{FNO} with alternative architectures, specifically a \textbf{convolutional neural network (CNN)} and a \textbf{recurrent neural operator(RNO)}~\cite{RNO}. While replacing FNO with RNO yields competitive performance, CNN-based models show considerable degradation, indicating that capturing nonlocal dependencies is essential for accurate flow control.}

\zelinrevision{Next, we investigate the impact of \textbf{multiplicative filters (MF)}, which enables the model to generalize across different Reynolds numbers. Removing MF leads to performance deterioration, particularly in high-Reynolds-number settings, emphasizing its role in learning a scalable control policy. Likewise, the \textbf{physics-informed learning ($L_{pde}$)} proves to be a crucial regularization mechanism, as removing it results in a consistent decline in drag reduction across all settings.}

\zelinrevision{We also explore an alternative formulation where \textbf{shear stresses replace pressure as the primary input feature}. This substitution leads to an overall decrease in performance, suggesting that pressure-based representations contain more informative signals for effective flow control. Furthermore, we examine the model's \textbf{robustness to noisy inputs} by introducing varying levels of Gaussian noise ($\frac{1}{SNR} = 0.05, 0.10, 0.20$). As noise increases, drag reduction performance degrades substantially, highlighting the sensitivity of the learned control policy to input uncertainty.}

\zelinrevision{Finally, the full model consistently outperforms all ablated versions, achieving the highest drag reduction across all generalization settings. These results underscore the necessity of each architectural and algorithmic component in achieving state-of-the-art flow control performance.}

\section{Acknowledgment}
Anima Anandkumar is supported by the Bren named chair professorship, Schmidt AI2050
senior fellowship, and ONR (MURI grant N00014-18-1-2624).

\section{Data Availability}
Source data are provided with this paper.

\section{Code Availability Statement}
The custom code used for the \modelname implementation and turbulent flow control simulations developed in this study is available from our GitHub repository \url{https://github.com/neuraloperator/pde-policylearning}. The code includes the neural operator architectures, training procedures, and evaluations. We provide a README file to explain how to use the code to reproduce our results.

\bibliography{sn-bibliography}

\end{document}